\documentclass[twocolumn]{article}
\usepackage{microtype}
\usepackage{graphicx}
\usepackage{booktabs} 
\usepackage{hyperref}
\usepackage{amsmath}
\usepackage{xcolor}
\usepackage{amsfonts}
\usepackage{amsthm}
\usepackage{fullpage}
\usepackage[smartEllipses,inlineFootnotes,hybrid]{markdown}
\usepackage{microtype}
\usepackage{graphicx}
\usepackage{booktabs} 
\usepackage{hyperref}
\usepackage{tikz}
\usepackage{graphicx}
\usepackage{caption}
\usepackage{subcaption}
\usepackage[numbers,sort]{natbib}
\usepackage{authblk}
\usepackage{xcolor}
\usepackage{amsfonts, amsmath, amsthm, amssymb}
\usepackage{algorithm, algorithmicx}
\usepackage{algpseudocode}
\usepackage{empheq}
\algnewcommand\Params{\item[\textbf{Parameters:}]}
\let\vec\relax
\DeclareMathAccent{\vec}{\mathord}{letters}{"7E}
\usepackage{accents}

\newcommand{\TODO}[1]{}

\def\A{{\mathcal A}}

\def\U{{\mathcal U}}
\def\M{{\mathcal M}}

\def\K{{\mathcal K}}

\def\F{{\mathcal F}}

\def\reals{{\mathbb R}}
\def\R{{\mathcal R}}

\newcommand{\ignore}[1]{}

\newcommand{\equaltri}{\triangleq}

\def\reals{{\mathbb R}}

\def\bold0{\mathbf{0}}

\def\epsilon{\varepsilon}

\def\R{\ensuremath{\mathcal R}}

\newcommand{\defeq}{\triangleq}

\newtheorem{theorem}{Theorem}[section]

\newtheorem{lemma}[theorem]{Lemma}

\newtheorem{definition}[theorem]{Definition}

\newtheorem{assumption}[theorem]{Assumption}

\numberwithin{equation}{section}

\def\horizon{T}
\def\gpch{L}
\def\gpcl{S}

\def\udiam{U}
\def\wdiam{W}
\def\Mdiam{\gamma}
\def\lindiam{\kappa}
\def\ustar{\accentset{\ast}{u}}
\def\Mstar{\accentset{\ast}{M}}

\title{A Regret Minimization Approach to Iterative Learning Control}
\author[1]{Naman Agarwal \thanks{namanagarwal@google.com}}
\author[1,2]{Elad Hazan \thanks{ehazan@cs.princeton.edu}}
\author[1,3]{Anirudha Majumdar \thanks{ani.majumdar@princeton.edu}}
\author[2]{Karan Singh \thanks{karans@princeton.edu}}
\affil[1]{Google AI Princeton}
\affil[2]{Department of Computer Science, Princeton University}
\affil[3]{Department of Mechanical and Aerospace Engineering, Princeton University}

\begin{document}
\maketitle

\begin{abstract}
We consider the setting of iterative learning control, or model-based policy learning in the presence of uncertain, time-varying dynamics. In this setting, we propose a new performance metric, \textit{planning regret}, which replaces the standard stochastic uncertainty assumptions with worst case regret. Based on recent advances in non-stochastic control, we design a new iterative algorithm for minimizing planning regret that is more robust to model mismatch and uncertainty. We provide theoretical and empirical evidence that the proposed algorithm outperforms existing methods on several benchmarks. 
\end{abstract}

\section{Introduction}
\label{sec:intro}



Consider a robotic system learning to perform a novel task, e.g., a quadrotor learning to fly to a specified goal, a manipulator learning to grasp a new object, or a fixed-wing airplane learning to perform a new maneuver. We are particularly interested in settings where (i) the task requires one to \emph{plan} over a given time horizon, (ii) we have access to an \emph{inaccurate} model of the world (e.g., due to unpredictable external disturbances such as wind gusts or misspecification of physical parameters such as masses, inertias, and friction coefficients), and (iii) the robot is allowed to iteratively refine its control policy via multiple executions (i.e., rollouts) on the real world. Motivated by applications where real-world rollouts are expensive and time-consuming, our goal in this paper is to learn to perform the given task as rapidly as possible. More precisely, given a cost function that specifies the task, our goal is to learn a low-cost control policy using a small number of rollouts.  

The problem described above is challenging due to a number of factors. The primary challenge we focus on in this paper is the existence of unmodeled deviations from nominal dynamics, and external disturbances acting on the system. Such disturbances may either be random or potentially even adversarial. 
In this paper we adopt a \emph{regret minimization} approach coupled with a recent  paradigm  called non-stochastic control to tackle this problem in generality. Specifically, consider a time-varying dynamical system given by the equation
\begin{equation} \label{eqn:dynamics}
 x_{t+1} = f_{t}(x_{t}, u_{t}) + w_{t} ,    
\end{equation}
where $x_{t}$ is the state, $u_t$ is the control input, and $w_{t}$ is an arbitrary disturbance at time $t$. Given a horizon $T$, the performance of a control algorithm $\mathcal{A}$ may be judged via the aggregate cost it suffers on a cost function sequence $c_1, \dots c_T$ along its state-action trajectory $(x_1^\mathcal{A}, u_1^\mathcal{A}, \dots)$:
$$ J(\mathcal{A} | w_{1:\horizon}) = \frac{1}{\horizon} \sum_{t=1}^{\horizon} c_t \left(x^{\mathcal{A}}_t, u_t^{\mathcal{A}}  \right).$$
For deterministic systems, an optimal open-loop control sequence $u_1\ldots u_{\horizon}$ can be chosen to minimize the cost sequence. The presence of unanticipated disturbances often necessitates the superposition of a \textit{closed-loop} correction policy $\pi$ to obtain meaningful performance. Such closed-loop policies can modify $u'_t = \pi(u_{1:t}, x_{1:t})$ as a function of the observed history thus far, and facilitate adaptation to realized disturbances. To capture this, we define a comparative performance metric, which we call {\bf Planning Regret}. In an episodic setting, for every episode $i$, an algorithm $\mathcal{A}$ adaptively selects control inputs while the rollout is performed under the influence of an arbitrary disturbance sequence $w_{1:\horizon}^{i}$. Planning regret is the difference between the total cost of the algorithm's actions and that of the retrospectively optimal open-loop plan coupled with episode-specific optimal closed-loop policies (from a policy class $\Pi$). Regret, therefore, is the relative cost of not knowing the to-be realized disturbances in advance. Formally for a total of $N$ rollouts, each of horizon $T$, it is defined as:
\begin{empheq}[box=\fbox]{align*}
&\text{\textbf{Planning Regret}}\\
\sum_{i=1}^{N} J(\mathcal{A} |  w^i_{1:\horizon}) - &\min_{u^\star_{1:\horizon} }  \sum_{i=1}^{N} \min_{\pi^\star_i \in \Pi} J(u^\star_{1:\horizon}, \pi^\star_i | w^t_{i,1:\horizon})
\end{empheq}

The motivation for our performance metric arises from the setting of Iterative Learning Control (ILC), where one assumes access to an imperfect (differentiable) simulator of real-world dynamics as well as access to a limited number of rollouts in the real world. In such a setting the disturbances capture the model-mismatch between the simulator and the real-world.
The main novelty in our formulation is the fact that, under vanishing regret, the closed-loop behavior of $\mathcal{A}$ is almost \textit{instance-wise optimal} on the specific trajectory, and therefore adapts to the passive controls, dynamics and disturbance for each particular rollout.
Indeed, worst-case regret is a stronger metric of performance than commonly considered in the planning/learning for control literature.

Our main result is an efficient algorithm that guarantees vanishing average planning regret for  non-stationary linear systems and disturbance-action policies. We experimentally demonstrate that the algorithm yields substantial improvements over ILC in linear and non-linear control settings. 

\paragraph{Paper structure.}
We present the relevant definitions including the setting in Section \ref{sec:setting}. 
The algorithm and the formal statement of the main result can be found in Section \ref{sec:mainresult}. In Section \ref{sec:overview} we provide an overview of the algorithm and the proof via the proposal of a more general and abstract \textit{nested online convex optimization (OCO) game}. This formulation can be of independent interest. Finally in Section \ref{sec:experiments}, we provide the results and details of the experiments.  Proofs and
other details are deferred to the Appendix.     






\subsection{Related Work}
\label{sec:related}

The literature on planning and learning in partially known MDPs is vast, and we focus here on the setting with the following characteristics:
\begin{enumerate}
\item 
We consider {\it model-aided} learning, which is suitable for situations in which the learner has some information about the dynamics, i.e. the mapping $f_t$ in Equation \eqref{eqn:dynamics}, but not the disturbances $w_t$. We further assume that we can differentiate through the model. This enables efficient gradient-based algorithms.

\item
We focus on the task of learning an episodic-agnostic control sequence, rather than a policy. This is aligned with the Pontryagin optimality principle \citep{pontryagin1962mathematical, ross2015primer}, and differs from dynamic programming approaches \citep{sutton2018reinforcement}.

\item
We accomodate arbitrary disturbance processes, and choose regret as a performance metric. This is a significant deviation from the literature on optimal and robust control \citep{kemin,stengel1994optimal}, and follows the lead of the recent paradigm of non-stochastic control \citep{agarwal2019online, hazan2019nonstochastic, simchowitz2020improper}. 

\item Our approach leverages multiple real-world rollouts. This access model is most similar to the iterative learning control (ILC) paradigm \citep{owens2005iterative,ahn2007iterative}.  For comparison, the model-predictive control (MPC) paradigm allows for only one real-world rollout on which performance is measured, and all other learning is permitted via access to a simulator.
\end{enumerate}

\paragraph{Optimal, Robust and Online  Control.} Classic results \citep{bertsekas2005dynamic, kemin, tedrake} in optimal control characterize the optimal policy for linear systems subject to i.i.d. perturbations given explicit knowledge of the system in advance.  Beyond stochastic perturbations, robust control  approaches \citep{zhou1998essentials} compute the best controller under worst-case noise. 

Recent work in machine learning \cite{abbasi2011regret,dean2018regret,mania2019certainty,cohen2018online,agarwal2019logarithmic} study regret bounds vs. the best linear controller in hindsight for online control with known and unknown linear dynamical systems. Online control was extended to adversarial perturbations, giving rise to the nonstochastic control model. In this general setting regret bounds were obtained for known/unknown systems as well as partial observation \citep{agarwal2019online,hazan2019nonstochastic,simchowitz2020improper,simchowitz2020making}. 

\paragraph{Planning with inaccurate models.}
Model predictive control (MPC) \citep{mayne2014model} provides a general scheme for planning with inaccurate models. MPC operates by applying model-based planning, (eg. iLQR \citep{li2004iterative, todorov2005generalized}), in a receding-horizon manner. MPC can also be extended to {robust} versions \citep{bemporad1999robust, mayne2005robust, langson2004robust} that explicitly reason about the parametric uncertainty or external disturbances in the model. Recently, MPC has also been viewed from the lens of online learning \citep{wagener2019online}. The setting we consider here is more general than MPC, allowing for iterative policy improvement across \emph{multiple rollouts} on the real world. 

An adjacent line of work on \textit{learning MPC} \cite{hewing2020learning, rosolia2017learning} focuses on constraint satisfaction and safety considerations while learning models simultaneously with policy execution. 

\paragraph{Iterative Learning Control (ILC).} ILC is a popular approach for tackling the setting considered. ILC operates by iteratively constructing a policy using an inaccurate model, executing this policy on the real world, and refining the policy based on the real-world rollout. ILC can be extended to use real-world rollouts to update the model (see, e.g., \cite{abbeel2006using}). For further details regarding ILC, we refer the reader to the text \cite{moore2012iterative}. Robust versions of ILC have also been developed in the control theory literature \citep{de1996synthesis},  using H-infinity control to capture bounded disturbances or uncertainty in the model. 

However, most of the work in robust control, typically account for \emph{worst-case} deviations from the model and can lead to extremely conservative behavior. In contrast, here we leverage the recently-proposed framework of \textit{non-stochastic control} to capture \emph{instance-specific} disturbances. We demonstrate both empirically and theoretically that the resulting algorithm provides significant gains in terms of sample efficiency over the standard ILC approach.

\paragraph{Meta-Learning.} Our setting, analysis and, in particular, the nested OCO setup bears similarity to formulations for gradient-based meta-learning (see \cite{finn2017model} and references therein). In particular, as we detail in the Appendix (Section \ref{sec:metacomparison}), the nested OCO setting we consider is a generalization of the setting considered in \cite{balcan2019provable}. We further detail certain improvements/advantages our algorithm and analysis provides over the results in \cite{balcan2019provable}. We believe this connection with Meta-Learning to be of independent interest.      


\section{Problem Setting} \label{sec:setting}

\subsection{Notation}

The norm $\|\cdot\|$ refers to the $\ell_2$ norm for vectors and spectral norm for matrices. For any natural number $n$, the set $[n]$ refers to the set $\{1,2 \ldots n\}$. We use the notation $v_{a:b} \defeq \{v_a \ldots v_b\}$ to denote a sequence of vectors/matrices. 
Given a set $S$, we use $v_{a:b} \in S$ to represent element wise inclusion, i.e. $\forall j \in [a,b], v_j \in S$; $\mathrm{Proj}_{S}(v_{a:b})$ represents the element-wise $\ell_2$ projection onto to the set $S$. $v_{a:b, c:d}$ denotes a sequence of sequences, i.e. $v_{a:b, c:d} = \{v_{a,c:d} \ldots v_{b,c:d}\}$ with $v_{a, c:d} = \{v_{a,c} \ldots v_{a,d}\}$. 

\subsection{Basic Definitions}

A \textbf{dynamical system} is specified via a start state $x_0\in \mathbb{R}^{d_x}$, a time horizon $\horizon$ and a sequence of transition functions $f_{1:\horizon}=\{f_t | f_t:\mathbb{R}^{d_x}\times \mathbb{R}^{d_u} \to \mathbb{R}^n\}$. The system produces a $\horizon$-length sequence of states $(x_1,\dots x_{\horizon+1})$ when subject to an $\horizon$-length sequence of actions $(u_1\dots u_{\horizon})$ and disturbances $\{w_1, \ldots w_{\horizon}\}$ according to the following dynamical equation\footnote{For the sake of simplicity, we do not consider a terminal cost, and consequently drop the last state from the description.}
\[ x_{t+1} = f_t(x_{t}, u_{t}) + w_{t}.\]
Through the paper the only assumption we make about the disturbance $w_{t}$ is that it is supported on a set of bounded diameter $W$. We assume full observation of the system, i.e. the states $x_{t}$ are visible to the controller. We also assume the passive transition function to be \textbf{known} beforehand. These assumptions imply that we fully observe the instantiation of the disturbances $w_{1:\horizon}$ during runs of the system.  

The actions above may be adaptively chosen based on the observed state sequence, ie. $u_{t} = \pi_{t}(x_1,\dots x_{t})$ for some non-stationary policy $\pi_{1:\horizon}=\{\pi_1,\dots \pi_{\horizon}\}$. We consider the policy to be deterministic (a restriction made for convenience). Therefore the state-action sequence $\{x_{t}, u_{t}\}_{t=1}^{\horizon}$, defined as $ x_{t+1} = f_t(x_t, u_t)+w_t, u_t = \pi_t(x_1 \ldots x_t)$, thus produced is a sequence determined by $w_{1:\horizon}$, fixing the policy, and the system.


A \textit{\textbf{rollout} of horizon $\mathbf{\horizon}$ on $f_{1:\horizon}$} refers to an evaluation of the above sequence for $\horizon$ time steps. When the dynamical system will be clear from the context, for the rest of the paper, we drop it from our notation. Given a cost function sequence $\{c_{t}\}:\mathbb{R}^n\times \mathbb{R}^d\to \mathbb{R}$ the \textit{\textbf{loss}} of executing a policy $\pi$ on the dynamical system $f$ with a particular disturbance sequence given by $w_{1:\horizon}$ is defined as
\[  J(\pi_{1:\horizon}|f_{1:\horizon}, w_{1:\horizon}) \defeq \frac{1}{\horizon} \left[\sum_{\tau=1}^{\horizon} c_{t}(x_{t}, u_{t})\right].\]

\begin{assumption}
\label{ass:cost}
  We will assume that the cost $c_{t}$ is a \textit{twice differentiable convex function} and that the value, gradient and hessian of the cost function $c_{t}$ is available. Further we assume,
  \begin{itemize}
      \item \textbf{Lipschitzness:} There exists a constant $G$ such that if $\|x\|,\|u\|\leq D$ for some $D > 0$,  then  
      $\|\nabla_{x}c_t(x,u)\|, \|\nabla_{u}c_t(x,u)\| \leq GD$. \item \textbf{Smoothness:} There exists a constant $\beta$ such that 
      for all $x,u$, $\nabla^2 c_t(x,u) \preceq \beta I$.
  \end{itemize}
\end{assumption}

When the dynamical system and the noise sequence are clear from the context we suppress them from the notation for the cost denoting it by $J(\pi_{1:\horizon})$. A particular sub-case which will be of special interest to us is the case of linear dynamical systems (LDS). Formally, a (non-stationary) linear dynamical system is described by a sequence of matrices $AB_{1:\horizon} = \{(A_t,B_t) \in \reals^{d_x,d_x} \times \reals^{d_x,d_u}\}_{t=1}^{\horizon}$ 
and the transition function is defined as $x_{t+1} = A_tx_t + B_tu_t$.

\begin{assumption}
\label{ass:linsystem}
We will assume that the linear dynamical system $AB_{1:\horizon}$ is  $(\lindiam,\delta)$- strongly stable for some $\lindiam > 0$ and $\delta \in (0,1]$, i.e. for if every $t$, we have that $\|A_t\| \leq 1-\delta, \|B_t\|\leq \lindiam.$
\end{assumption}

We note that all the results in the paper can be easily generalized to a weaker notion of strong stability where the linear dynamical system  is  $(\lindiam,\delta)$- strongly stable if there exists a sequence of matrices $K_{1:\horizon}$, such that for every $t$, we have that $\|A_t-B_tK_t\| \leq 1-\delta, \|B_t\|,\|K_t\| \textbf{}\leq \kappa.$ A system satisfying such an assumption can be easily transformed to a system satisfying Assumption \ref{ass:linsystem} by setting $A_t = A_t - B_tK_t$. This redefinition is equivalent to appending the linear policy $K_t$ on top of the policy being executed. While we present the results for the case when $K_T = 0$, the only difference the non-zero case makes to our analysis is potentially increasing the norm of the played actions which  can still be shown to be bounded. Overall this nuance leads to a difference to our main result only in terms of factors polynomial in the system parameters. Hence for convenience, we state our results under Assumption \ref{ass:linsystem}. The assumption of strong-stability (in a weaker form as allowed by stationary systems) has been popular in recent works on online control \citep{cohen2018online,agarwal2019online} and the above notion generalizes it to non-stationary systems. 

\subsection{Policy Classes}

\paragraph{Open-Loop Policies.} Given a convex set $\U \in \reals^{d_u}$, consider a sequence of control actions, $u_{1:\horizon} \in \U$. We define (by an overload of notation), the open-loop policy $u_{1:\horizon}$ as a policy which plays at time $t$, the action $u_{t}$. The set of all such policies is defined as $\Pi_{\U} \defeq \U^{\otimes \horizon}$. 

Given two policies we define the sum of the two (denoted by $\pi_1 + \pi_2$) as the policy for which the action at time $t$ is the sum of the action recommended by policy $\pi_1$ and $\pi_2$. 

\paragraph{Linear Policies.}
Given a matrix $K \in \reals^{d_u, d_x}$, a \textit{linear policy} \footnote{For notational simplicity, we do not include an affine offset $c_t$ in the definition of our linear policy; this can be included with no change in results across the paper.} denoted (via an overload of notation) by  $K$ is a policy that plays action $u_{t} = Kx_{t}$. Such linear state-feedback policies are known to be optimal for the LQR problem and for $H_\infty$ control \citep{kemin}.

\paragraph{Disturbance Action Policies.} A generalization of the class of linear policies can be obtained via the notion of disturbance-action policies (see \cite{agarwal2019online}) defined as follows. A disturbance action policy  $\pi_{M_{1:\gpch}}$ of memory length $\gpch$ is defined by a sequence of matrices $M_{1:\gpch} \defeq \{M_1 \ldots M_{\gpch}\}$ where each $M_i \in \M \subseteq \{\reals^{d_u \times d_x}\}$, with the action at time step $t$ given by
\begin{equation}
\label{eq:da}
    \pi_{M_{1:\gpch}} \defeq \sum_{j=1}^{\gpch} M_j w_{t-j}.
\end{equation}
A natural class of matrices from which the above feedback matrices can be picked is given by fixing a number $\Mdiam > 0$ and picking matrices spectrally bounded by $\Mdiam$, i.e. $\M_{\Mdiam} \defeq \{M | M \in \reals^{d_u \times d_x}, \|M\| \leq \Mdiam\}$. We further overload the notation for a disturbance action policy to incorporate an open-loop control sequence $u_{1:\horizon}$, defined as $ \pi_{M_{1:\gpch}}(u_{1:\horizon}) \defeq u_t + \sum_{j=1}^{\gpch} M_j w_{t-j}$. 

\subsection{Planning Regret With Disturbance-Action Policies}

As discussed, a natural idea to deal with adversarial process disturbance is to plan (potentially oblivious to it), producing a sequence of open loop ($u_{1:\horizon}$) actions and appending an adaptive controller to \textit{correct} for the disturbance online. However the disturbance in practice could have structure across rollouts, which can be leveraged to improve the plan ($u_{1:\horizon}$), with the knowledge that we have access to an adaptive controller. To capture this, we define the notion of an online planning game and the associated notion of planning regret below.  
\begin{definition}[Online Planning]
\label{def:onlineplanning}
It is defined as an N round/rollout game between a player and an adversary, with each round defined as follows: 
\begin{itemize}
\vspace{-3pt}
    \item At every round $i$ the player given the knowledge of a new dynamical system $f^i_{1:\horizon} = \{f^i_1 \ldots f^{i}_{\horizon}\}$, proposes a policy $\pi^{i}_{1:\horizon} =  \{\pi^i_1 \ldots \pi^{i}_{\horizon}\}$. 
    \item The adversary then proposes a noise sequence $w_{1:\horizon}^i$ and a cost sequence $c_{1:\horizon}^i$. 
    \vspace{-3pt}
    \item A rollout of policy $\pi^{i}_{1:\horizon}$ is performed on the system $f^{i}_{1:\horizon}$ with disturbances $w^{i}_{1:\horizon}$ and the cost suffered by the player $J_i(\pi^i_{1:\horizon}) \defeq J(\pi^i_{1:\horizon} | f^i_{1:\horizon}, w_{1:\horizon}^i)$. 
\end{itemize}
\end{definition}

The task of the controller is to minimize the cost suffered. We measure the performance of the controller via the following objective, defined as \textbf{Planning-Regret}, which measures the performance against the metric of producing the best-in-hindsight open-loop plan, having been guaranteed the optimal adaptive control policy for every single rollout. The notion of adaptive control policy we use is the disturbance-action policy class defined in \eqref{eq:da}. In the Appendix (Section \ref{lincompare}), we discuss the expressiveness of the disturbance-actions policies. In particular, they generalize linear policies for stationary systems and lend convexity. Formally, planning regret is defined as follows:
\begin{empheq}[box=\fbox]{align*}
&\text{\textbf{Planning Regret}}  \\\sum_{i=1}^{N} J_i(\pi^i_{1:\horizon}) - &\min_{u_{1: \horizon}} \sum_{i=1}^{N} \left( \min_{M_{1:\gpch}} J_i\left(\pi_{M_{1:\gpch}}(u_{1:\horizon})\right)\right)
\end{empheq}

\section{Main Algorithm and Result} \label{sec:mainresult}

In this section we propose the algorithm \textbf{iGPC} (Iterative Gradient Perturbation Controller; Algorithm \ref{alg:bigalg}) to minimize Planning Regret. The algorithm at every iteration given an open-loop policy $u_{1:\horizon}$ performs a rollout overlaying an online DAC adaptive controller GPC (Algorithm \ref{alg:gpcoverlay}). Further the base policy $u_{1:\horizon}$ is updated by performing gradient descent (or any other local policy improvement) on $u$ fixing the offsets suggested by GPC. \footnote{In Appendix Section \ref{app:generalalgos}, we provide a more general version of the algorithm defined for any base policy class.} We show the following guarantee on average planning regret for Algorithm \ref{alg:bigalg} for linear dynamical systems.

\begin{theorem}
\label{thm:main}
Let $\mathcal{U} \subseteq \reals^{d_u}$ be a bounded convex set with diameter $\udiam$. Consider the online planning game (Definition \ref{def:onlineplanning}) with linear dynamical systems $\{ AB_{1:\horizon}^{i} \}_{i=1}^{N}$ satisfying Assumption \ref{ass:linsystem} and cost functions $\{c_{1:\horizon}\}_{i=1}^N$ satisfying Assumption \ref{ass:cost}. Then we have that Algorithm \ref{alg:bigalg} (when executed with appropriate parameters), for any sequence of disturbances $\{w^i_{1:\horizon}\}_{i=1}^{N}$ with each $\|w^{i}_t\| \leq \wdiam$ and any $\Mdiam \geq 0$, produces a sequence of actions with planning regret bounded as
\begin{align*}
    \frac{1}{N}& \left(\sum_{i=1}^{N} J_i(\pi^i_{1:\horizon}) \right. \\& \left.- \min_{u_{1: \horizon} \in \U} \left(\sum_{i=1}^{N}  \min_{M_{1:\gpch}\in \M_{\Mdiam}}  J_i\left(\pi_{M_{1:\gpch}}(u_{1:\horizon})\right)\right)\right) \\ & \leq \tilde{O}\left(\frac{1}{\sqrt{T}} + \frac{1}{\sqrt{N}}\right).
\end{align*}
where $\M_{\Mdiam} = \{M | M \in \reals^{d_u, d_x}, \|M\| \leq \Mdiam\}$.   
\end{theorem}

The $\tilde{O}$ notation above subsumes factors polynomial in system parameters $\lindiam, \Mdiam,  \delta^{-1}, \udiam, \wdiam, G$ and $\log(T)$. A restatement of the theorem with all the details is present in the Appendix (Section \ref{app:proof}).

\begin{algorithm}[h!]
\caption{iGPC Algorithm}
\label{alg:bigalg}
\begin{algorithmic}[1]
\Require [Online] $f_{1:\horizon}^{1:N}:$ Dynamical Systems, $w_{1:\horizon}^{1:N}:$ Disturbances , $c_{1:\horizon}^{1:N}$
\Params Set : $\U$, $\eta_{\mathrm{out}}:$ Learning Rate
\vspace{5pt}
\State Initialize $u_{1:\horizon}^1 \in \U$ arbitrarily.
\For{$i = 1\ldots N$}
\State Receive a dynamical system $f_{1:\horizon}^i$. 
\State \textbf{Rollout} $u^i_{1:\horizon}$ with GPC \Comment{(Alg. \ref{alg:gpcoverlay})},
\[ \{ x^i_{1:\horizon}, a^i_{1:\horizon}, w^i_{1:\horizon}, o_{1:\horizon}^i\} = \mathrm{GPCRollout}(f_{1:\horizon}^i, u^i_{1:\horizon}). \] 
\State \textbf{Update}: Compute the update to the policy,
\begin{align*}
    &\nabla_i =   \nabla_{u_{1:\horizon}} J( u_{1:\horizon}^{i}+ o_{1:\horizon}^{i} | f_{1:\horizon}^i, w_{1:\horizon}^i) \\
    &u_{1:\horizon}^{i+1} = \mathrm{Proj}_{\U} \left( u_{1:\horizon}^{i} - \eta_{\mathrm{out}} \nabla_i \right).
\end{align*}
\EndFor
\end{algorithmic}
\end{algorithm}
\begin{algorithm}[h!]
\caption{GPCRollout}
\label{alg:gpcoverlay}
\begin{algorithmic}[1]
\Require $f_{1:\horizon}$ : dynamical system, $u_{1:\horizon}$ : input policy, [Online] $w_{1:\horizon}$ : disturbances, $c_{1:\horizon}:$ costs. 
\Params  $\gpch$ : Window, $\eta_{\mathrm{in}}$ : Learning rate, $\Mdiam$ : Feedback bound, $\gpcl$ : Lookback
\vspace{5pt}
\State Initialize $M_{1,1:\gpch} = \{M_{1,j}\}_{j=1}^{\gpch} \in \M_{\Mdiam}$. 
\State Set $w_i = 0$ for all $i \leq 0$.
\For{$t=1 \ldots \horizon$}
\State Compute Offset: $o_t = \sum_{r=1}^{\gpch} M_{t,r} \cdot w_{t-r}$. 
\State Play action: $a_{t} = u_t + o_t$.
\State Suffer Cost: $c_t(x_t, a_t)$ 
\State Observe state: $x_{t+1}$.
\State Compute perturbation: \[w_{t} = x_{t+1} - f_t(x_t, a_t).\]
\State Do a gradient step on the $\mathrm{GPCLoss}$ \eqref{eqn:gpcloss}
{\small \begin{align*}
     &M_{t+1, 1:\gpch} = \mathrm{Proj}_{\M_{\Mdiam}} \left(M_{t, 1:\gpch} -\eta_{\mathrm{in}} \nabla \mathrm{GPCLoss}(\textit{arg})\right),
\end{align*}}
where $\textit{arg}$ captures policy $M_{t,1:\gpch}$, open-loop plan $u_{t-\gpcl+1:t}$, disturbances $w_{t-\gpcl-\gpch+1:t-1}$, transition $f_{t-\gpcl + 1, t-1},$ cost $c_t$ in Equation~\ref{eqn:gpcloss} and gradient is taken with respect to the $M$ parameter.
\EndFor
\State \Return $x_{1:\horizon}, a_{1:\horizon}, w_{1:\horizon}$, $o_{1:\horizon}$.
\end{algorithmic}
\end{algorithm}

\section{Algorithm and Analysis}
\label{sec:overview}
In this section we provide an overview of the derivation of the algorithm and the proof for Theorem \ref{thm:main}. The formal proof is deferred to Appendix (Section \ref{app:proof}). We introduce an online learning setting that is the main building block of our algorithm. The setting applies more generally to control/planning and our formulation of planning regret in linear dynamical systems is a specification of this setting.

\vspace{-5pt}
\subsection{Nested OCO and Planning Regret}
 \paragraph{Setting:}
Consider an online convex optimization(OCO) problem \citep{OCObook}, where the iterations have a nested structure, divided into inner and outer iterations. Fix two convex sets $\K_1$ and $\K_2$. After every one out of $N$ outer iterations, the player chooses a point  $x_i \in \K_1$. After that there is a sequence of $T$ inner iterations, where the player chooses $y_t^i  \in \K_2 $ at every iteration. After this choice, the adversary chooses a convex cost function $f_{t}^{i} \in \F \subseteq \K_1 \times \K_2 \rightarrow \reals$, and the player suffers a cost of $f_{t}^{i}(x_i,y_t^i)$. The goal of the player is to minimize Planning Regret:
\begin{empheq}[box=\fbox]{align*}
&\text{\textbf{Planning Regret}}\\
\sum_{i=1}^{N} \sum_{t=1}^T \frac{f_{t}^{i}(x_i,y_t^{i})}{T}  &- \min_{x^\star \in \K_1  } \sum_{i=1}^{N} \min_{y^\star \in \K_2 } \sum_{t=1}^{T}  \frac{ f_{t}^{i}(x^\star,y^\star)}{T} 
\end{empheq}

To state a general result, we assume access to two online learners denoted by $\A_1,\A_2$, that are guaranteed to provide sub-linear regret bounds over \textit{linear} cost functions on the sets $\K_1, \K_2$ respectively in the standard OCO model. We denote the corresponding regrets achieved by $R_N(\A_1), R_T(\A_2)$. A canonical algorithm for online linear optimization (OLO) is online gradient descent \citep{zinkevich2003online}, which is what we use in the sequel. The theory presented here applies more generally. \footnote{Regret for OLO depends on function bounds, which correspond to gradient bounds here. For clarity we omit this dependence from the notation for regret.}
Algorithm \ref{alg:nestedalg} lays out a general algorithm for the Nested-OCO setup. 
\begin{algorithm}[h!]
\caption{Nested-OCO Algorithm}
\label{alg:nestedalg}
\begin{algorithmic}[1]
\Require Algorithms $\A_1,\A_2$. \State Initialize $x_1 \in \K_1$ arbitrarily.
\For{$i = 1\ldots N$}
\State Initialize $y_0^{i} \in \K_2$ arbitrarily. 
\For{$t = 1\ldots T$}
\State Define loss function over $\K_2$ as\[ h_{t}^{i}(y)  \equaltri \nabla_y f_{t}^{i}(x_i,y_{t}^{i}) \cdot y. \]
\State Update $y_{t+1} \leftarrow \A_2(h_{0}^{i} \ldots h_{t}^{i}) $.
\EndFor
\State Define loss function over $\K_1$ as \[g_{i}(x)  \equaltri  \sum_{t=1}^T \nabla_x f_{t}^{i}(x_i,y_t^{i}) \cdot x .\]
\State Update $x_{s+1} \leftarrow \A_1(g_1,...,g_i) $.
\EndFor
\end{algorithmic}
\end{algorithm}
\begin{theorem}
\label{thm:simplenested}
Algorithm \ref{alg:nestedalg} with sub-algorithms $\A_1,\A_2$ with regrets $R_N(\A_1),R_T(\A_2)$ ensures the following regret guarantee on the average planning regret, 
$$ \frac{\mathrm{Planning Regret}}{N} \leq \frac{R_N(\A_1)}{N} +  \frac{R_{T}(\A_2)}{\horizon} .$$
\end{theorem}
When using Online Gradient Descent as the base algorithm, the average regret scales as 
$ O\left(\frac{1}{\sqrt{N}} +  \frac{1}{\sqrt{T}}\right) $.
\begin{proof}[Proof of Theorem \ref{thm:simplenested}]
Let $x^{\star} \in \K_1$ be any point and let $y^{\star}_{1:T} \in \K_2$ be any sequence. We have
\begin{align*}
&\frac{\sum_{i=1}^{N}   \sum_{t=1}^{T}  f_{t}^{i}(x_i,y_t^{i}) -  f_{t}^{i}(x^\star,y^\star_i)}{TN} \\
\leq & \frac{\sum_{i=1}^{N}   \sum_{t=1}^{T} \nabla_x f_{t}^{i} (x_i - x^\star)}{TN}  \\
& \qquad\qquad + \frac{ \sum_{i=1}^{N}   \sum_{t=1}^{T} \nabla_y f_{t}^{i} (y_t^{i} - y^\star  )}{TN}\\
 =& \frac{\sum_{i=1}^{N} [ g_i (x_i)  - g_i(x^\star)]}{TN}  \\
 & \qquad\qquad + \frac{\sum_{i=1}^{N}   \sum_{t=1}^{T} [h_{t}^{i}(y_t) - h_{t}^{i}( y^\star_i  ) ]}{TN}  \\\leq & \frac{R_N(\A_1)}{N} + \frac{ \cdot R_{T}(\A_2)}{T}, 
\end{align*}
where the first inequality follows by convexity and the last inequality follows by the regret guarantees and noting that the functions $g_i$ are naturally scaled up by a factor of $T$. 
\end{proof}

\subsection{Proof Sketch for Theorem \ref{thm:main}}

The main idea behind the proof is to reduce to the setting of Theorem \ref{thm:simplenested}. In the reduction the $x$ variable corresponds to the open loop controls $u_{1:\horizon} \in \U$ and the variables $y_{t}^{i}$ correspond to the closed-loop disturbance-action policy $M_{t,1:\gpch}^{i} \in \M_{\Mdiam}$. The algorithms $\A_1$ and $\A_2$ are instantiated as Online Gradient Descent with appropriately chosen learning rates. 

We begin the reduction by using the observation in \cite{agarwal2019online} that costs are convex with respect to the variables $u,M$, for \textit{linear dynamical systems} with convex costs. With convexity, prima-facie the reduction seems immediate, however this is impeded by the counterfactual notion of policy regret which implies that cost at any time is dependent on previous actions. This nuance in the reduction from Theorem \ref{thm:simplenested} is only applicable to the closed loop policies $M$, the open loop part $u_{1:\horizon}$) on the other hand, follows according to the reduction and hence direct OGD is applied (Line 6, Algorithm \ref{alg:bigalg}).  

To resolve the issue of the counterfactual dependence, we use the techniques introduced in the OCO with memory framework proposed by \cite{anava2015online} and recently employed in the work of \cite{agarwal2019online}. We leverage the underlying stability of the dynamical system to ensure that cost at time $t$ depends only on a bounded number of previous rounds, say $\gpcl$. We then define a proxy loss denoted by $\mathrm{GPCLoss}$, corresponding to the cost incurred by a stationary closed-loop policy executing for the previous $\gpcl$ time steps. Formally, given a dynamical system $f_{1:\gpcl}$, perturbations $w_{1:\gpcl}$, a cost function $c$, a non-stationary open-loop policy $u_{1:\gpcl}$, $\mathrm{GPCLoss}$ is a function of closed-loop transfer $M_{1:L}$ defined as follows. Consider the following iterations with $y_1 = 0$,
\begin{align}
    a_j \defeq u_{j} &+ \sum_{r=1}^{\gpch} M_r w_{j-r}, \nonumber \\ y_{j} \defeq f_{j-1}  (y_{j-1},& a_{j-1}) + w_{j-1} \quad \forall j \in [1,\gpcl], \nonumber \\ 
    \mathrm{GPCLoss}(M_{1:\gpch},  u_{1:\gpcl}, &w_{-\gpch+1:\gpcl-1}, f_{1:\gpcl-1}, c) \defeq c(y_\gpcl, a_\gpcl). \label{eqn:gpcloss}
\end{align}
The algorithm updates by performing a gradient descent step on this loss, i.e. $M_{t+1,1:\gpch}^{i} = M_{t,1:\gpch}^{i} - \eta \nabla_{M} \mathrm{GPCLoss}(\cdot)$. The proof proceeds by showing that the actual cost and \text{its gradient} is closely tracked by their proxy GPC Loss counterparts with the difference proportional to the learning rate (Appendix Lemma \ref{lemma:gradcomparison}). Choosing the learning rate appropriately then completes the proof.


\section{Experiments}
\label{sec:experiments}
We demonstrate the efficacy of the proposed approach on two sets of experiments: the theory-aligned one performs basic checks on linear dynamical systems; the subsequent set demonstrates the benefit on highly non-linear systems  distilled from practical applications. In the following we provide a detailed description of the setup and the results are presented in Figure \ref{fig:quad_wind}.

\subsection{Experimental Setup}

We briefly review the methods that we compare to: The \textbf{ILQG} agent obtains a closed loop policy via the Iterative Linear Quadratic Gaussian algorithm \citep{todorov2005generalized}, proposed originally to handle Gaussian noise while planning on non-linear systems, on the simulator dynamics, and then executes the policy thus obtained. This approach does not \textit{learn} from multiple rollouts and, if the dynamics are fixed, provides a constant (across rollouts) baseline. 

The Iterative Learning Control (\textbf{ILC}) agent \citep{abbeel2006using} \textit{learns} from past trajectories to refine its actions on the next real-world rollout. We provide precise details in the Appendix (Section \ref{app:allalgoscomb}). Finally, the \textbf{IGPC} agent adapts Algorithm \ref{alg:bigalg} by replacing the policy update step (Line 5) with a LQR step on locally linearized dynamics.

In all our experiments, the metric we compare is the number of real-world rollouts required to achieve a certain loss value on the real dynamics. For further details on the setups and hyperparameter tuning please see Appendix (Section \ref{app:allalgoscomb}).

\subsection{Linear Control}
This section considers a discrete-time \textbf{Double Integrator} (detailed below), a basic kinematics model well studied in control theory. This linear system (described below) is subject to a variety of perturbations that vary either within or across episodes, 
\begin{align*}
    A = \begin{bmatrix}
    1 & 1 \\ 0 & 1
    \end{bmatrix}\quad  B = \begin{bmatrix} 0 \\ 1\end{bmatrix}.
\end{align*}

We pick three instructive perturbation models: First, as a sanity check, we consider constant offsets. While both ILC and IGPC adapt to this change, IGPC is quicker in doing so as evident by the cost on the first rollout itself. In the second, we treat constant offsets that gradually increase in magnitude from zero with rollouts/episodes. While gradual inter-episodic changes are well suited to ILC, IGPC still offers consistently better performance. The final scenario considers time-varying sinusoidal perturbations subject to rollout-varying phase shifts. In contrast to the former setups, such conditions make intra-episodic learning crucial for good performance. Indeed, IGPC outperforms alternatives here by a margin, reflecting the benefit of rollout-adaptive feedback policy in the regret bound.

\subsection{Non-linear Control with Approximate Models}
Here, we consider the task of controlling non-linear systems whose real-world characteristics are only partially known. In the cases presented below, the proposed algorithm \textbf{IGPC} either converges to the optimal cost with fewer rollouts (for Quadrotor), or, even disregarding speed of convergence, offers a better terminal solution quality (for Reacher). These effects are generally more pronounced in situations where the model mismatch is severe.

Concretely, consider the following setup: the agent is scored on the cost incurred on a handful of sequentially executed real-world rollouts on a dynamical system $g(x,u)$; all the while, the agent has access to an inaccurate simulator $f(x,u)\neq g(x,u)$. In particular, while limited to simply observing its trajectories in the real world $g$, the agent is permitted to compute the function value and Jacobian of the simulator $f(x,y)$ along arbitrary state-action pairs. The disturbances here are thus the difference between $g$ and $f$ along the state-action pairs visited along any given real world rollout. Here, we also consider a statistically-omnipotent \textit{infeasible agent} \textbf{ILQR (oracle)} that executes the Iterative Linear Quadratic Regulator algorithm \citep{li2004iterative} directly via Jacobians of the real world dynamics $g$ (a cheat), indicating a lower bound on the best possible cost.

\paragraph{Quadrotor with Wind}
The simulator models an underactuated planar quadrotor (6 dimensional state, 2 dimensional control) attempting to fly to $(1, 1)$ from the origin. The real-world dynamics differ from the simulator in the presence of a dispersive force field ($x \hat{\bf i}+y \hat{\bf j}$), to accomodate wind. The cost is measured as the distance sqaured from the origin along with a quadratic penalty on the actions.

\begin{figure*}[h!]
\centering
\includegraphics[width=\textwidth]{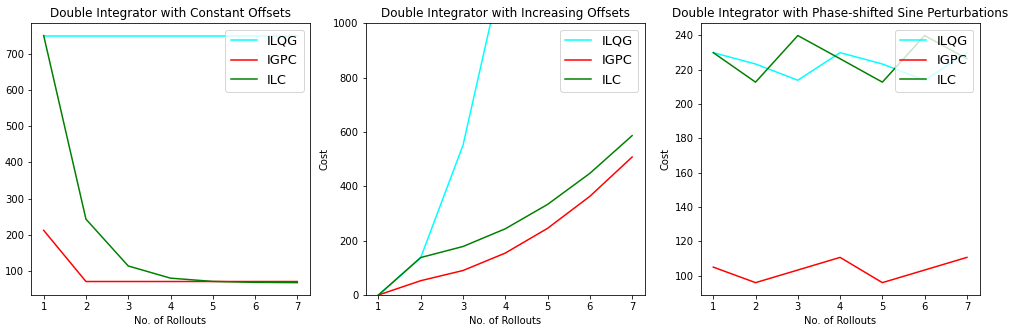}\\
\includegraphics[width=\textwidth]{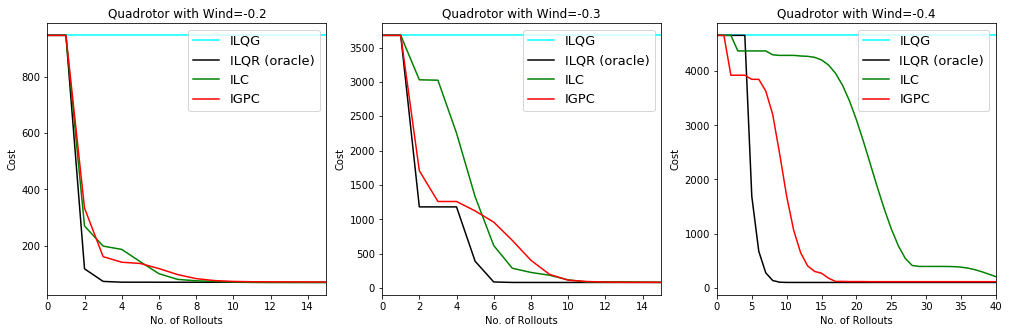}\\
\includegraphics[width=\textwidth]{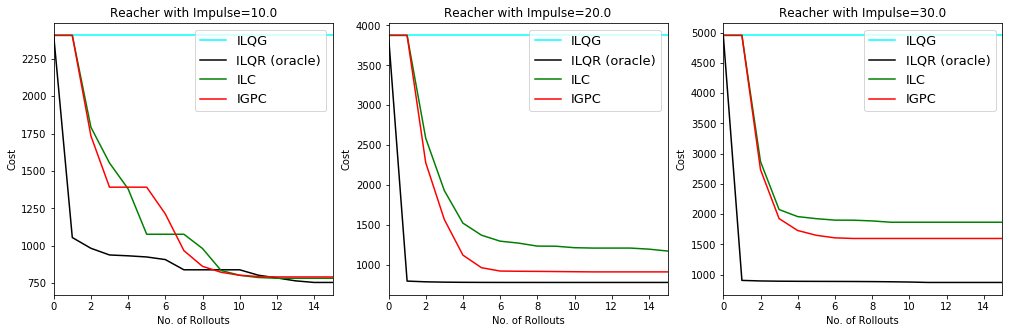}
\
\caption{On top is a linear system, Double Integrator, setup subject to: (L) constant offset, (M) offset that increments with rollout count, (R) phase-shifted sinusoidal perturbations. The middle section displays results on the quadrotor environment for varying magnitudes of wind. Bottom figure captures performance on the reacher environment with varying magnitudes of periodic impulses. \textbf{ILQR (oracle)} is an infeasible agent with access to Jacobians on the real world.}\label{fig:quad_wind}
\end{figure*}
\paragraph{Reacher with Impulse}
The simulator dynamics model a 2-DOF arm (6 dimensional state, 2 dimensional control) attempting to place its end-effector at a pre-specified goal. The true dynamics $g$ differs from the simulator in the application of periodic impulses to the center of mass of the arm links. The cost involves a quadratic penalty on the controls and the distance of the end effector from the goal.

In both  scenarios, JAX-based \cite{jax} differentiable implementations of the underlying dynamics were adapted from  \citep{gradu2021deluca}. 

\section{Conclusion}
In this work, we cast the task of disturbance-resilient planning into a regret minimization framework. We outline a gradient-based algorithm that refines an open loop plan in conjunction with a near instance-optimal closed loop policy. We provide a theoretical justification for the approach by proving a vanishing average regret bound. We also demonstrate our approach on simulated examples and observe empirical gains compared to the popular iterative learning control (ILC) approach. 

There are several exciting directions for future work. On the theoretical front, one interesting direction is to establish lower bounds on regret for the setting we consider here. This may provide an indication of how sub-optimal (in terms of regret) the approach we present here is and potentially guide the way towards improved algorithms. On the algorithmic front, extending our approach to handle partially observable settings would be of significant practical interest (e.g., settings where there is a mismatch in the sensor model in addition to the dynamics model). Finally, a particular exciting direction is to experimentally explore the benefits in terms of sim-to-real transfer conferred by our approach. 
\clearpage
\onecolumn

\bibliography{main}

\newcommand{\etalchar}[1]{$^{#1}$}
\begin{thebibliography}{DMM{\etalchar{+}}18}

\bibitem[ABH{\etalchar{+}}19]{agarwal2019online}
Naman Agarwal, Brian Bullins, Elad Hazan, Sham Kakade, and Karan Singh.
\newblock Online control with adversarial disturbances.
\newblock In {\em Proceedings of the 36th International Conference on Machine
  Learning}, pages 111--119, 2019.

\bibitem[ACM07]{ahn2007iterative}
Hyo-Sung Ahn, YangQuan Chen, and Kevin~L Moore.
\newblock Iterative learning control: Brief survey and categorization.
\newblock {\em IEEE Transactions on Systems, Man, and Cybernetics, Part C
  (Applications and Reviews)}, 37(6):1099--1121, 2007.

\bibitem[AHM15]{anava2015online}
Oren Anava, Elad Hazan, and Shie Mannor.
\newblock Online learning for adversaries with memory: price of past mistakes.
\newblock In {\em Advances in Neural Information Processing Systems}, pages
  784--792, 2015.

\bibitem[AHS19]{agarwal2019logarithmic}
Naman Agarwal, Elad Hazan, and Karan Singh.
\newblock Logarithmic regret for online control.
\newblock {\em arXiv preprint arXiv:1909.05062}, 2019.

\bibitem[AQN06]{abbeel2006using}
Pieter Abbeel, Morgan Quigley, and Andrew~Y Ng.
\newblock Using inaccurate models in reinforcement learning.
\newblock In {\em Proceedings of the 23rd international conference on Machine
  learning}, pages 1--8. ACM, 2006.

\bibitem[AYS11]{abbasi2011regret}
Yasin Abbasi-Yadkori and Csaba Szepesv{\'a}ri.
\newblock Regret bounds for the adaptive control of linear quadratic systems.
\newblock In {\em Proceedings of the 24th Annual Conference on Learning
  Theory}, pages 1--26, 2011.

\bibitem[Ber05]{bertsekas2005dynamic}
Dimitri Bertsekas.
\newblock {\em Dynamic programming and optimal control}, volume~1.
\newblock Athena scientific Belmont, MA, 2005.

\bibitem[BFH{\etalchar{+}}18]{jax}
James Bradbury, Roy Frostig, Peter Hawkins, Matthew~James Johnson, Chris Leary,
  Dougal Maclaurin, George Necula, Adam Paszke, Jake Vander{P}las, Skye
  Wanderman-{M}ilne, and Qiao Zhang.
\newblock {JAX}: composable transformations of {P}ython+{N}um{P}y programs,
  2018.

\bibitem[BKT19]{balcan2019provable}
Maria-Florina Balcan, Mikhail Khodak, and Ameet Talwalkar.
\newblock Provable guarantees for gradient-based meta-learning.
\newblock In {\em International Conference on Machine Learning}, pages
  424--433. PMLR, 2019.

\bibitem[BM99]{bemporad1999robust}
Alberto Bemporad and Manfred Morari.
\newblock Robust model predictive control: A survey.
\newblock In {\em Robustness in identification and control}, pages 207--226.
  Springer, 1999.

\bibitem[CHK{\etalchar{+}}18]{cohen2018online}
Alon Cohen, Avinatan Hasidim, Tomer Koren, Nevena Lazic, Yishay Mansour, and
  Kunal Talwar.
\newblock Online linear quadratic control.
\newblock In {\em International Conference on Machine Learning}, pages
  1028--1037, 2018.

\bibitem[DMM{\etalchar{+}}18]{dean2018regret}
Sarah Dean, Horia Mania, Nikolai Matni, Benjamin Recht, and Stephen Tu.
\newblock Regret bounds for robust adaptive control of the linear quadratic
  regulator.
\newblock In {\em Advances in Neural Information Processing Systems}, pages
  4188--4197, 2018.

\bibitem[dR96]{de1996synthesis}
Dick de~Roover.
\newblock Synthesis of a robust iterative learning controller using an
  h/sub/spl infin//approach.
\newblock In {\em Proceedings of 35th IEEE Conference on Decision and Control},
  volume~3, pages 3044--3049. IEEE, 1996.

\bibitem[FAL17]{finn2017model}
Chelsea Finn, Pieter Abbeel, and Sergey Levine.
\newblock Model-agnostic meta-learning for fast adaptation of deep networks.
\newblock {\em arXiv preprint arXiv:1703.03400}, 2017.

\bibitem[GHS{\etalchar{+}}21]{gradu2021deluca}
Paula Gradu, John Hallman, Daniel Suo, Alex Yu, Naman Agarwal, Udaya Ghai,
  Karan Singh, Cyril Zhang, Anirudha Majumdar, and Elad Hazan.
\newblock Deluca--a differentiable control library: Environments, methods, and
  benchmarking.
\newblock {\em arXiv preprint arXiv:2102.09968}, 2021.

\bibitem[Haz16]{OCObook}
Elad Hazan.
\newblock Introduction to online convex optimization.
\newblock {\em Foundations and Trends in Optimization}, 2(3-4):157--325, 2016.

\bibitem[HKS20]{hazan2019nonstochastic}
Elad Hazan, Sham Kakade, and Karan Singh.
\newblock The nonstochastic control problem.
\newblock In Aryeh Kontorovich and Gergely Neu, editors, {\em Proceedings of
  the 31st International Conference on Algorithmic Learning Theory}, volume 117
  of {\em Proceedings of Machine Learning Research}, pages 408--421, San Diego,
  California, USA, 08 Feb--11 Feb 2020. PMLR.

\bibitem[HWMZ20]{hewing2020learning}
Lukas Hewing, Kim~P Wabersich, Marcel Menner, and Melanie~N Zeilinger.
\newblock Learning-based model predictive control: Toward safe learning in
  control.
\newblock {\em Annual Review of Control, Robotics, and Autonomous Systems},
  3:269--296, 2020.

\bibitem[LCRM04]{langson2004robust}
Wilbur Langson, Ioannis Chryssochoos, SV~Rakovi{\'c}, and David~Q Mayne.
\newblock Robust model predictive control using tubes.
\newblock {\em Automatica}, 40(1):125--133, 2004.

\bibitem[LT04]{li2004iterative}
Weiwei Li and Emanuel Todorov.
\newblock Iterative linear quadratic regulator design for nonlinear biological
  movement systems.
\newblock In {\em ICINCO (1)}, pages 222--229, 2004.

\bibitem[May14]{mayne2014model}
David~Q Mayne.
\newblock Model predictive control: Recent developments and future promise.
\newblock {\em Automatica}, 50(12):2967--2986, 2014.

\bibitem[Moo12]{moore2012iterative}
Kevin~L Moore.
\newblock {\em Iterative learning control for deterministic systems}.
\newblock Springer Science \& Business Media, 2012.

\bibitem[MSR05]{mayne2005robust}
David~Q Mayne, Mar{\'\i}a~M Seron, and SV~Rakovi{\'c}.
\newblock Robust model predictive control of constrained linear systems with
  bounded disturbances.
\newblock {\em Automatica}, 41(2):219--224, 2005.

\bibitem[MTR19]{mania2019certainty}
Horia Mania, Stephen Tu, and Benjamin Recht.
\newblock Certainty equivalent control of lqr is efficient.
\newblock {\em arXiv preprint arXiv:1902.07826}, 2019.

\bibitem[OH05]{owens2005iterative}
David~H Owens and Jari H{\"a}t{\"o}nen.
\newblock Iterative learning control—an optimization paradigm.
\newblock {\em Annual reviews in control}, 29(1):57--70, 2005.

\bibitem[PBGM62]{pontryagin1962mathematical}
Lev~S Pontryagin, VG~Boltyanskii, RV~Gamkrelidze, and EF~Mishchenko.
\newblock The mathematical theory of optimal processes, translated by kn
  trirogoff.
\newblock {\em New York}, 1962.

\bibitem[RB17]{rosolia2017learning}
Ugo Rosolia and Francesco Borrelli.
\newblock Learning model predictive control for iterative tasks. a data-driven
  control framework.
\newblock {\em IEEE Transactions on Automatic Control}, 63(7):1883--1896, 2017.

\bibitem[Ros15]{ross2015primer}
I~Michael Ross.
\newblock {\em A primer on Pontryagin's principle in optimal control}.
\newblock Collegiate publishers, 2015.

\bibitem[SB18]{sutton2018reinforcement}
Richard~S Sutton and Andrew~G Barto.
\newblock {\em Reinforcement learning: An introduction}.
\newblock MIT press, 2018.

\bibitem[Sim20]{simchowitz2020making}
Max Simchowitz.
\newblock Making non-stochastic control (almost) as easy as stochastic.
\newblock {\em arXiv preprint arXiv:2006.05910}, 2020.

\bibitem[SSH20]{simchowitz2020improper}
Max Simchowitz, Karan Singh, and Elad Hazan.
\newblock Improper learning for non-stochastic control, 2020.

\bibitem[Ste94]{stengel1994optimal}
Robert~F Stengel.
\newblock {\em Optimal control and estimation}.
\newblock Courier Corporation, 1994.

\bibitem[Ted20]{tedrake}
Russ Tedrake.
\newblock {\em Underactuated Robotics: Algorithms for Walking, Running,
  Swimming, Flying, and Manipulation (Course Notes for MIT 6.832)}.
\newblock 2020.

\bibitem[TL05]{todorov2005generalized}
Emanuel Todorov and Weiwei Li.
\newblock A generalized iterative lqg method for locally-optimal feedback
  control of constrained nonlinear stochastic systems.
\newblock In {\em Proceedings of the 2005, American Control Conference, 2005.},
  pages 300--306. IEEE, 2005.

\bibitem[WCSB19]{wagener2019online}
Nolan Wagener, Ching-An Cheng, Jacob Sacks, and Byron Boots.
\newblock An online learning approach to model predictive control.
\newblock {\em arXiv preprint arXiv:1902.08967}, 2019.

\bibitem[ZD98]{zhou1998essentials}
Kemin Zhou and John~Comstock Doyle.
\newblock {\em Essentials of robust control}, volume 104.
\newblock Prentice hall Upper Saddle River, NJ, 1998.

\bibitem[ZDG96]{kemin}
Kemin Zhou, John~C. Doyle, and Keith Glover.
\newblock {\em Robust and Optimal Control}.
\newblock Prentice-Hall, Inc., USA, 1996.

\bibitem[Zin03]{zinkevich2003online}
Martin Zinkevich.
\newblock Online convex programming and generalized infinitesimal gradient
  ascent.
\newblock In {\em Proceedings of the 20th International Conference on Machine
  Learning (ICML-03)}, pages 928--936, 2003.

\end{thebibliography}
\bibliographystyle{alpha}

\newpage
\appendix

\section{Relationship with Meta-Learning}
\label{sec:metacomparison}
In this section, we detail how the nested-OCO formulation proposed in the paper can be used to derive upto a small constant factor, the gradient based meta-learning results presented in \cite{balcan2019provable} by reducing their setting to the nested-OCO setting and applying Algorithm \ref{alg:nestedalg}. The reduction requires setting the $x,y$ space, i.e. $\K_1, \K_2$ to be $\Theta\subseteq \R^d$ and a ball in $\R^d$ of diameter $D^*$ (according to the notation in \cite{balcan2019provable}). Further, we set the function $f_{t}^i(x,y) \defeq l_{t,i}(x+y)$. The reduction recovers the same guarantee as the result in \cite{balcan2019provable} upto a factor of 2. 

We note that  \cite{balcan2019provable} provide an algorithm that works without the knowledge of $D^*$, but such an extension is standard in OCO literature and can be handled similarly to \cite{balcan2019provable}. Further we acknowledge that for the particular problem considered in \cite{balcan2019provable}, the constant factor is important as  a straightforward algorithm also achieves the same rate if constant factors are ignored, a fact highlighted in the original paper. On the other hand, our formulation allows for a stronger comparator even in the \cite{balcan2019provable} setup. 

We would like to highlight that our nested-OCO setup allowing for different $x,y$ spaces is more general than the setup typically considered in initialization-based meta-learning. Owing to this generality, the algorithm we provide naturally performs a gradient step on the true function value for the outer loop as opposed to a distance based function as in \cite{balcan2019provable}. Further exploring the effectiveness of our algorithm for meta-learning is left as interesting future work.

\section{Comparison of Policy Classes}\label{lincompare}
In this section we make a comparison of various policy classes introduced in the paper. 

\paragraph{Linear state-action policies.} 
In classical optimal control with full observation, the cost function is typically assumed to be quadratic in the state and control, i.e. 
$$ c_t(x,u) = x^\top Q x + u^\top R u .$$ 
Under this assumption and infinite horizon time-invariant $(A_i,B_i=A_j,B_j)$ linear dynamical system (LDS), and assuming independent Gaussian disturbances at every time step, the optimal solution can be computed using the Bellman optimality equations (see e.g. \cite{tedrake}). This gives rise to the Discrete time Algebraic Riccati Equation (DARE), whose solution is a linear policy commonly denoted by 
$$ u_t = K x_t . $$
The finite-horizon solution is also computable and results in a non-stationary linear policy, where the linear policies converge exponentially fast to the first solution of the Riccati equation. It is thus reasonable to consider the class of all linear policies as a reasonable comparator class. Denote the class of all linear policies as
$$ \Pi_L = \{ K \in \reals^{d_x \times d_u } \} . $$

\paragraph{State of the art: linear dynamical control policies. }

A generalization of static state-action control policies is that of linear dynamical controllers (LDC). LDC are particularly useful for partially observed LDS and maintain their own internal dynamical system according to the observations in order to recover the hidden state of the system.  A formal definition is below. 

\begin{definition}[Linear Dynamic Controllers] \label{def:ldc}
A linear dynamic controller $\pi$ is a linear dynamical system $(A_\pi, B_\pi, C_\pi, D_\pi)$ with internal state $s_t\in \mathbb{R}^{d_\pi}$, input $x_t\in \mathbb{R}^{d_x}$ and output $u_t\in\mathbb{R}^{d_u}$ that satisfies
$$
s_{t+1} = A_\pi s_t + B_\pi x_t,\ \ u_t = C_\pi s_t + D_\pi x_t.
$$
\end{definition}
LDC are state-of-the-art in terms of performance and prevalence in control applications involving LDS, both in the full and partial observation settings. They are known to be theoretically optimal for partially observed LDS with quadratic cost functions and normally distributed noise, but are more widely used. Denote the class of all LDC as
$$ \Pi_{LDC} = \{ A \in \reals^{d_s \times d_s} ,B \in \reals^{d_s \times d_x}, C \in \reals ^{d_u \times d_s} ,D \in \reals^{d_u \times d_x } \} . $$

\paragraph{Disturbance-Action Controllers (DAC)}

As we have defined earlier, we consider an even
more general class of policies, i.e. that of disturbance-action control. For linear time invariant systems, this policy class is more general than that of LDC and linear controllers, in the sense that for every LDS there exists a DAC which outputs exactly the same controls on the same system and sequence of noises. With a finite and fixed $H$, an approximate version of this statement is true.
The precise approximation statement and formal proof can be found in \cite{agarwal2019online}. A similar statement can be made for LDC as well. 

However we note that all of the above statements hold only in linear time invariant case. In the time varying case, these generalizations are not necessarily true, however note that we are using disturbance action feedback control only as an adaptive control policy to correct against noise, and it is added upon an open-loop plan.

\section{Main Theorem and Proof}
\label{app:proof}

We provide the following restatement of Theorem \ref{thm:main} with details regarding the parameters and the dependence on the system parameters. To state the results concisely, we assume that all the appropriate assumed constants, i.e. $\lindiam,\Mdiam,G,\beta,\udiam,\wdiam$ are greater than 1. This is done to upper bound the sum of two constants by twice their product. All the results hold by replacing any of these constants by the max of the constant and 1. 

\begin{theorem}
\label{thm:mainrestated}
Let $\mathcal{U} \subseteq \reals^{d_u}$ be a bounded convex set with diameter $\udiam$. Consider the online planning game(Definition \ref{def:onlineplanning}) with linear dynamical systems $\{ AB_{1:\horizon}^{i} \}_{i=1}^{N}$ satisfying Assumption \ref{ass:linsystem} and cost functions $\{c_{1:\horizon}\}_{i=1}^N$ satisfying Assumption \ref{ass:cost}. Then we have that Algorithm \ref{alg:bigalg}(when executed with appropriate parameters), for any sequence of disturbances $\{w^i_{1:\horizon}\}_{i=1}^{N}$ with each $\|w^{i}_t\| \leq \wdiam$ and any $\Mdiam \geq 0$, produces a sequence of actions with planning regret bounded as
\[\frac{1}{N} \left(\sum_{i=1}^{N} J_i(\pi^i_{1:\horizon}) - \min_{u_{1: \horizon} \in \U} \left(\sum_{i=1}^{N}  \min_{M_{1:\gpch}\in \M_{\Mdiam}} J_i\left(\pi_{M_{1:\gpch}}(u_{1:\horizon})\right)\right)\right)  \leq \left(\frac{c_{\mathrm{in}}\log^2(T)}{\sqrt{T}} + \frac{c_{\mathrm{out}}}{\sqrt{N}}\right).\]
where $\M_{\Mdiam} = \{M | M \in \reals^{d_u, d_x}, \|M\| \leq \Mdiam\}$and $c_{\mathrm{in}},c_{\mathrm{out}}$ are constants depending on system parameters as follows 

\[ c_{\mathrm{out}} = \tilde{O}\left(G\udiam(\udiam +  \Mdiam\gpch \wdiam)\kappa^{2}\delta^{-2}\right)\]
\[ c_{\mathrm{in}} = \tilde{O}\left(\sqrt{ \Mdiam^3 \lindiam^4 \delta^{-3} \beta G^2 \gpch^5 \wdiam^3(\udiam + \Mdiam \gpch \wdiam)^2}\right).\]
Here $\tilde{O}$ subsumes constant factors and factors poly-logarithimic in the the arguments of $\tilde{O}$.
To achieve the above bound, 
Algorithm \ref{alg:bigalg} is to be executed with parameters, learning rate $\eta_{\mathrm{out}} = \frac{\udiam}{G\lindiam\delta^{-2}(\lindiam\udiam  +  \lindiam\Mdiam \gpch \wdiam + \wdiam) \sqrt{N}} $, with the inner execution of Algorithm \ref{alg:gpcoverlay} is performed with parameters $\eta_{\mathrm{in}} = \frac{\gamma^2 \gpch^2}{\sqrt{ 12\Mdiam \lindiam^4 \delta^{-5} \beta G^2 \gpch^3 \wdiam^3(\udiam + \Mdiam \gpch \wdiam)^2}}$ and $\gpcl = \delta^{-1}\log(\eta_{\mathrm{in}})$. 
\end{theorem}

\subsection{Requisite Definitions}

Before proving the theorem we set up some useful definitions. Fix a linear dynamical system $AB_{1:\horizon}$ and a disturbance sequence $w_{1:\horizon}$. For any sequence $u_{1:\horizon} \in \U$ and $M_{1:\horizon, 1:\gpch} \in \M_{\Mdiam}$, we define $\horizon$ functions $x_{1:\horizon}(\cdot | AB_{1:\horizon},w_{1:\horizon}),a_{1:\horizon}(\cdot | AB_{1:\horizon},w_{1:\horizon})$, denoting the action played and the state visited at time $t$ upon execution of the policies together. Herein we drop $AB_{1:\horizon},w_{1:\horizon}$ from the notation when clear from the context. Formally, consider the following definitions for all $t$,  

\begin{equation}
	\label{eqn:adef}
	a_t(u_{1:\horizon}, M_{1:\horizon, 1:\gpch}) \defeq u_{t} + \sum_{r=1}^{\gpch} M_{t,r} w_{t-r}
\end{equation}
\begin{equation}
\label{eqn:xdef}
	x_1(u_{1:\horizon}, M_{1:\horizon, 1:\gpch}) \defeq 0 \qquad x_{t+1}(u_{1:\horizon}, M_{1:\horizon, 1:\gpch}) \defeq  A_t x_{t}(u_{1:\horizon}, M_{1:\horizon, 1:\gpch}) + B_t a_t + w_t
\end{equation}
Given a sequence of cost functions $c_{1:\horizon}(x, u): \reals^{d_{x} \times d_{u}} \rightarrow \reals$, satisfying Assumption \ref{ass:cost}, define via an overload of notation, the cost functions $c_t$ as a function of $u_{1:\horizon}, M_{1:\horizon,1:\gpch}$ as follows
\begin{equation}
\label{eqn:cdef}
    \forall t \in [1:\horizon], \qquad c_t(u_{1:t},M_{1:\horizon, 1:\gpch}) = c_t(x_t(u_{1:t},M_{1:\horizon, 1:\gpch}), a_t(u_{1:\horizon}, M_{1:\horizon, 1:\gpch}))
\end{equation}
Naturally, according to our definition of the total cost $J$ of the rollout we get that 
\[J(u_{1:\textbf{}\horizon}, M_{1:\horizon,1:\gpch}) = \frac{1}{\horizon}\sum_{t=1}^{\horizon} c_t(u_{1:t},M_{1:\horizon, 1:\gpch})\]
Next, we expand upon the recursive definition of $x_t(\cdot, \cdot)$ via the following operators,
\begin{definition} 
\label{def:linops}
Given a linear dynamical system $AB_{1:\horizon}$, define the following transfer matrices 
\[ \forall j \in [\horizon],\forall k \in [j+1,\horizon]\quad  T_{j \rightarrow k} \in \reals^{d_x \times d_u} \qquad T_{j \rightarrow k} \defeq \begin{cases}
I & \text{if $k=j+1$} \\
\left(\Pi_{t=j+2}^{k} A_t\right) & \text{otherwise} 
\end{cases}\]
Additionally given a disturbance sequence $w_{1:\horizon}$, define the following linear operator over matrix sequences $M_{1:\horizon, 1:\gpch}$ 
\[ \forall j \in [\horizon], \forall k \in [j+1,\horizon] \quad \psi^M_{j \rightarrow k}: [\reals^{d_u \times d_x}]^{\horizon \times \gpch} \rightarrow \reals^{d_x}\]
\[\psi^M_{j \rightarrow k}(M_{1:\horizon, 1:\gpch}) = \sum_{t=j}^{k-1} \left( T_{t \rightarrow k}B_t \left(\sum_{r=1}^{\gpch} M_{t,r} w_{k-r} \right)\right)\]
\end{definition}
It can be observed via unrolling the recursion and the definitions above that 
\begin{equation}
\label{eqn:xexpansion}
	x_t(u_{1:\horizon},M_{1:\horizon,1:\gpch}) = \sum_{j=1}^{t-1} T_{j \rightarrow t} (B_ju_j + w_j) + \psi_{1\rightarrow t}^{M}(M_{1:\horizon, 1:\gpch}).
\end{equation}
Since $x_t,a_t$ are linear functions of $u_{1:\horizon},M_{1:\horizon,1:\gpch}$, therefore we have that $c_t(u_{1:\horizon},M_{1:\horizon,1:\gpch})$ is a convex function of its arguments. The next lemma further shows that the gradient of the total cost with respect to the argument $u_{1:T}$ is bounded, as stated in the following lemma.
\begin{lemma}
\label{lemma:gradbound}
Given a linear system $AB_{1:\horizon}$ satisfying Assumption \ref{ass:linsystem}, a bounded disturbance sequence $w_{1:\horizon}$ and a cost sequence $c_t$ satisfying Assumption \ref{ass:cost}, then for any $\Mdiam \geq 0, \U$, let $u_{1:\horizon} \in \U, M_{1:\horizon, 1:\gpch} \in \M_{\Mdiam}$ be two sequences, then we have that
\[ \left\|\nabla_{u_j}\left( \sum_{t=1}^{\horizon}  c_t(u_{1:\horizon}, M_{1:\horizon,1:\gpch}) \right)\right\| \leq 2G\lindiam\delta^{-2}(\lindiam\udiam  +  \lindiam\Mdiam \gpch \wdiam + \wdiam) \]
\end{lemma}
We provide the proof of the lemma further in the section. Using the lemma we are now ready to prove Theorem \ref{thm:main}. 
\begin{proof}[Proof of Theorem \ref{thm:main}]
Lets fix a particular rollout $i$. Let $AB_{1:\horizon}^{i}$ be the dynamical system and $w^{i}_{1:\horizon}$ be the disturbance supplied. Further  $u_{1:\horizon}^{i}$ be the open loop control sequence played at round $i$ and $M_{1:\horizon, 1:\gpch}^{i}$ be the disturbance feedback sequence played by the GPC subroutine. By definition we have that the state achieved 
\[ x^{i}_t = x_t(u_{1:\horizon}^{i}, M_{1:\horizon, 1:\gpch}^{i}) \qquad a^{i}_t = a_t(u_{1:\horizon}^{i}, M_{1:\horizon, 1:\gpch}^{i})\]
We have for convenience dropped the system and disturbance from our notation. The total cost at round $i$ incurred by the algorithm by definition is 
\[ J = \sum_{i=1}^{N} \frac{1}{\horizon}\left(\sum_{t=1}^{\horizon} c_t^{i}(u_{1:\horizon}^{i}, M^i_{1:\horizon,1:\gpch})\right)\]
Fix the sequence of comparators to be $\ustar_{1:\horizon}, \{\Mstar_{1:\gpch}^{i}\}_{i=1}^{N}$. The comparator cost by definition then is
\[ \accentset{\ast}{J} = \sum_{i=1}^{N} \frac{1}{\horizon}\left(\sum_{t=1}^{\horizon} c_t^{i}(\ustar_{1:\horizon}, \mathcal{T}_{\horizon}\Mstar^i_{1:\gpch})\right),\]
where given a sequence $v_{a:b}$, we define the tiling operator $\mathcal{T}_k$, which creates a nested sequence of outer length $k$ by tiling with copies of the sequence $v_{a:b}$, i.e. $\mathcal{T}_kv_{a:b}=[v_{a:b},v_{a:b} \ldots v_{a:b}]$. 
We therefore have the following calculation for the regret which follows from the convexity of the cost function $c_t$ with respect to $u,M$ as established before,
\begin{align*}
    &\sum_{i=1}^{N} \sum_{t=1}^{\horizon} \left( c_t^{i}(u_{1:\horizon}^{i}, M^i_{1:\horizon,1:\gpch}) - c_t^{i}(\ustar_{1:\horizon}, \mathcal{T}_{\horizon}\Mstar^i_{1:\gpch})\right) \\
    &\leq \sum_{i=1}^{N} \sum_{t=1}^{\horizon} \left( \nabla_u c_t^{i}(u_{1:\horizon}^{i}, M^i_{1:\horizon,1:\gpch}) (u_{1:\horizon}^{i} - \ustar_{1:\horizon}) + \nabla_M c_t^{i}(u_{1:\horizon}^{i}, M^i_{1:\horizon,1:\gpch}) (M^i_{1:\horizon,1:\gpch} - \Mstar^i_{1:\gpch}) \right) \\
    &= \underbrace{\sum_{i=1}^{N} \sum_{t=1}^{\horizon} \left( \nabla_u c_t^{i}(u_{1:\horizon}^{i}, M^i_{1:\horizon,1:\gpch}) (u_{1:\horizon}^{i} - \ustar_{1:\horizon}) \right)}_{\text{Outer Regret}} + \underbrace{\sum_{i=1}^{N} \sum_{t=1}^{\horizon} \left( \nabla_M c_t^{i}(u_{1:\horizon}^{i}, M^i_{1:\horizon,1:\gpch}) (M^i_{1:\horizon,1:\gpch} - \Mstar^i_{1:\gpch}) \right)}_{\text{Inner Regret}} 
\end{align*}
We analyze the both the terms above separately. We begin by analyzing the first term. 
\paragraph{Outer Regret:}
Consider the following calculation
\begin{align*}
   \sum_{i=1}^{N} \sum_{t=1}^{\horizon} \left( \nabla_u c_t^{i}(u_{1:\horizon}^{i}, M^i_{1:\horizon,1:\gpch}) (u_{1:\horizon}^{i} - \ustar_{1:\horizon}) \right) = \sum_{j=1}^{\horizon} \sum_{i=1}^{N}   \underbrace{\nabla_{u_j}\left( \sum_{t=1}^{\horizon}  c_t^{i}(u_{1:\horizon}^{i}, M^i_{1:\horizon,1:\gpch}) \right)}_{\defeq g^u_{ij}} (u_j^{i} - \ustar_j). 
\end{align*}
Note that by definition of the algorithm, we have that for all $i,j$ 
\[ u_j^{i+1} = \mathrm{Proj}_{\U}(u_j^{i} - \eta_{\mathrm{out}} g_{ij}^{u}),\]
which via the pythagorean inequality implies that
\[\|u_j^{i+1} - \ustar_{j}\|^2 \leq \|u_j^{i} - \eta_{\mathrm{out}} g_{ij}^{u} - \ustar_{j}\|^2\]
Combining the above equations we immediately get that
\begin{align}
\sum_{i=1}^{N} \sum_{t=1}^{\horizon} \left( \nabla_u c_t^{i}(u_{1:\horizon}^{i}, M^i_{1:\horizon,1:\gpch}) (u_{1:\horizon}^{i} - \ustar_{1:\horizon}) \right)  
   &\leq \sum_{j=1}^{\horizon} \sum_{i=1}^{N} \frac{1}{2} \left( \eta_{\mathrm{out}} \|g^u_{ij}\|^2 + \frac{(u_j^{i} - \ustar_j)^2 - (u_j^{i+1} - \ustar_j)^2}{\eta_{\mathrm{out}} }\right)  \nonumber
   \\&\leq \sum_{j=1}^{\horizon} \frac{1}{2} \left(  \eta_{\mathrm{out}}  \left(\sum_{i=1}^{N} \|g^u_{ji}\|^2\right) + \frac{(u_j^{1} - \ustar_j)^2}{\eta_{\mathrm{out}} }\right) \nonumber
   \\ &\leq 2\udiam G\lindiam\delta^{-2}(\lindiam\udiam  +  \lindiam\Mdiam \gpch \wdiam + \wdiam)T\sqrt{N}  \label{eqn:InnerRegret}
\end{align}
where the last inequality follows using Lemma \ref{lemma:gradbound} and choice of $\eta_{\mathrm{out}}$. 

\paragraph{Inner Regret:}
Next we analyze the second Inner Regret term.
Before doing so we recommend the reader to re-familiarize with the notations defined in Definition \ref{def:linops} and Equations \ref{eqn:adef},\ref{eqn:xdef},\ref{eqn:cdef}. We will also need the following further definitions again for a fixed rollout. Therefore given a dynamical system $AB_{1:\horizon}$, a disturbance sequence $w_{1:\horizon}$, and an open loop sequence $u_{1:\horizon}$ define the notion of surrogate state at time $t$ which is parameterized by a lookback window $\gpcl$ and is a function of an input sequence $M_{1:\gpch} \in \reals^{d_u \times d_x}$. Intuitively it corresponds to the state achieved by executing the stationary policy $M_{1:\gpch}$ along with $u_{1:T}$ for $S$ time steps, starting at time $t-S$ with a resetted state. This is exactly the computation performed in the GPCLoss definition in Equation \ref{eqn:gpcloss}. We can use the linear operator $\psi$ defined in Definition \ref{def:linops} for an alternative and succinct definition as follows. 
\begin{equation}
\label{eqn:sxdef}
    \hat{x}_t(u_{1:\horizon}, M_{1:\gpch}) = \sum_{j=t-\gpcl}^{t-1} T_{j\rightarrow t} (B_ju_{j} + w_j) +  \psi^M_{t-\gpcl \rightarrow t}(\mathcal{T}_{\horizon} M_{1:\gpch}).
\end{equation}
 Further given a cost function $c_t$, we can use the above definition to also define a surrogate cost
\begin{equation}
\label{eqn:scdef}
    \hat{c}_t(u_{1:\horizon}, M_{1:\gpch}) = c_t\left(\hat{x}_t(u_{1:\horizon}, M_{1:\gpch}), u_t + \sum_{j=1}^{\gpch} M_{j}w_{t-j}\right)
\end{equation}
It can be observed now by the definition of Algorithm \ref{alg:gpcoverlay}, the sequence $M^{i}_{1:\horizon,1:\gpch}$ played by the algorithm is chosen iteratively as follows
\begin{equation}
\label{eqn:algorecursion}
    M^{i}_{{t+1},1:\gpch} = \mathrm{Proj}_{\M_{\Mdiam}}\left(M^i_{{t},1:\gpch} - \eta_{\mathrm{in}} \nabla_{M} \hat{c}_t(u_{1:\horizon}^i, M^i_{{t},1:\gpch})\right).
\end{equation}
To proceed with the proof we will need the following lemma 

\begin{lemma}
\label{lemma:gradcomparison}
Consider a linear system $AB_{1:\horizon}$ satisfying Assumption \ref{ass:linsystem}, a bounded disturbance sequence $w_{1:\horizon}$ and a sequence of cost functions $c_{1:\horizon}$ satisfying Assumption \ref{ass:cost}. Given any open loop sequence $u_{1:\horizon} \in \U$ and a closed-loop matrix sequence $M_{1:\horizon,1:\gpch} \in \M_{\Mdiam}$ generated through the iteration specified in Equation \ref{eqn:algorecursion}, we have that the following properties hold for all $t \in [\horizon]$
\begin{itemize}
    \item For all $j > t$, $\nabla_{M_{j,1:\gpch}} c_t(u_{1:\horizon}, M_{1:\horizon,1:\gpch}) = 0$.
    \item For all $j < t$, $\|\nabla_{M_{j,1:\gpch}} c_t(u_{1:\horizon}, M_{1:\horizon,1:\gpch})\| \leq \lindiam^2 G (\udiam + \Mdiam \gpch \wdiam)\gpch \wdiam \delta^{-1} (1-\delta)^{t-j}$.
    \item For all $t$, $\|\nabla_{M_{1:\gpch}} \hat{c}_t(u_{1:\horizon}, M_{1:\gpch})\| \leq G \gpch \wdiam(\udiam + \Mdiam \gpch \wdiam)\left(1 + \frac{\lindiam^2}{\delta^2}\right)$.
    \item Furthermore, for any $\Mstar_{1:\gpch} \in \M_{\Mdiam}$ and for any $t$, we have that
\begin{align*}
    \sum_{j=t-\gpcl}^t \nabla_{M_{j,1:\gpch}} c_t(u_{1:\horizon}, M_{1:\horizon,1:\gpch}) (M_{j,1:\gpch} - \Mstar_{1:\gpch})  
&\leq \nabla_{M_{t, 1:\gpch}} \hat{c}_t(u_{1:\horizon}, M_{t,1:\gpch}) (M_{t,1:\gpch} - \Mstar_{1:\gpch}) \\  &+ 20\eta_{\mathrm{in}}\log^2(\eta_{\mathrm{in}}) \Mdiam \lindiam^4 \delta^{-3} \beta G^2 \gpch^3 \wdiam^3(\udiam + \Mdiam \gpch \wdiam)^2
\end{align*}
\end{itemize}
\end{lemma}

We are now ready to analyze the inner regret term. We analyze this term for one particular rollout say $i$ (thereby dropping $i$ from our notation). We get the following series of calculations,

\begin{align*}
    &\sum_{t=1}^{\horizon} \left( \nabla_{M} c_t(u_{1:\horizon}, M_{1:\horizon,1:\gpch}) (M_{1:\horizon,1:\gpch} - \mathcal{T}_{T}\Mstar_{1:\gpch}) \right) \\
    &=  \sum_{t=1}^{\horizon}\sum_{j=1}^{\horizon}  \left( \nabla_{M_{j,1:\gpch}} c_t(u_{1:\horizon}, M_{1:\horizon,1:\gpch}) (M_{j,1:\gpch} - \Mstar_{1:\gpch}) \right) \\
    &=  \sum_{t=1}^{\horizon}\sum_{j=1}^{t}  \left( \nabla_{M_{j,1:\gpch}} c_t(u_{1:\horizon}, M_{1:\horizon,1:\gpch}) (M_{j,1:\gpch} - \Mstar_{1:\gpch}) \right) \\
    &\leq \sum_{t=1}^{\horizon}\sum_{j=t-\gpcl}^{t}  \left( \nabla_{M_{j,1:\gpch}} c_t(u_{1:\horizon}, M_{1:\horizon,1:\gpch}) (M_{j,1:\gpch} - \Mstar_{1:\gpch}) \right) + 2\lindiam^2 \Mdiam  G\gpch \wdiam(\udiam + \Mdiam \gpch \wdiam) \delta^{-2} (1-\delta)^{\gpcl} \\
    &\leq \sum_{t=1}^{\horizon}  \left(\underbrace{ \nabla_{M_{t,1:\gpch}} \hat{c}_t(u_{1:\horizon}, M_{t,1:\gpch})}_{g_t} (M_{t,1:\gpch} - \Mstar_{1:\gpch}) \right) + 22\horizon\eta_{\mathrm{in}}\log^2(\eta_{\mathrm{in}}) \Mdiam \lindiam^4 \delta^{-3} \beta G^2 \gpch^3 \wdiam^3(\udiam + \Mdiam \gpch \wdiam)^2, 
    \end{align*}
    where the statements follow via repeated application of Lemma \ref{lemma:gradcomparison} and the choice of $\gpcl = \delta^{-1}\log(\eta_{\mathrm{in}})$. To analyse further once again via a similar argument as in the case of the outer regret regarding projected gradient descent with learning rate $\eta_{\mathrm{in}}$, we get that,
    \begin{align*}
    &\sum_{t=1}^{\horizon}  \left(\underbrace{ \nabla_{M_{t,1:\gpch}} \hat{c}_t(u_{1:\horizon}, M_{1:\horizon,1:\gpch})}_{g_t} (M_{t,1:\gpch} - \Mstar_{1:\gpch}) \right) \\
    &\leq \sum_{t=1}^{T} \left( \frac{\eta_{\mathrm{in}}}{2}\|g_t\|^2 + \frac{\|M_{t,1:\gpch} - \Mstar_{1:\gpch}\|^2 - \|M_{t+1,1:\gpch} - \Mstar_{1:\gpch}\|^2}{2\eta_{\mathrm{in}}} \right)\\
    & \leq  \frac{\eta_{\mathrm{in}} T}{2}\|g_t\|^2 + \frac{\|M_{1,1:\gpch} - \Mstar_{1:\gpch}\|^2}{2\eta_{\mathrm{in}}} 
\end{align*}
Combining the above equations, Equation \ref{eqn:costgradbound} and the choice of $\eta_{\mathrm{in}}$, we get that the inner regret is bounded as,
\begin{align*}
\sum_{t=1}^{\horizon} \left( \nabla_{M} c_t(u_{1:\horizon}, M_{1:\horizon,1:\gpch}) (M_{1:\horizon,1:\gpch} - \mathcal{T}_{T}\Mstar_{1:\gpch}) \right) \leq 
\tilde{O}\left(\sqrt{\horizon \Mdiam^3 \lindiam^4 \delta^{-3} \beta G^2 \gpch^5 \wdiam^3(\udiam + \Mdiam \gpch \wdiam)^2}\right)
\end{align*} 
Combining the outer and inner regret terms we finish the proof. 
\end{proof}















In the remaining subsections we prove Lemmas \ref{lemma:gradbound} and \ref{lemma:gradcomparison}, thereby finishing the proof of Theorem \ref{thm:main}.

\subsection{Proof of Lemma \ref{lemma:gradbound}}
In this section we prove Lemma \ref{lemma:gradbound}. Before the proof we establish some other lemmas which will be useful to us.
\begin{lemma}
\label{lem:Tbound}
Given a linear system $AB_{1:\horizon}$ satisfying Assumption \ref{ass:linsystem}, then the transfer matrices defined in Definition \ref{def:linops} are bounded as follows
\[ \forall j,k \in [T],[j+1,T] \qquad \|T_{j\rightarrow k}\| \leq  (1 - \delta)^{k-j-1}\]
\end{lemma}
\begin{proof}[Proof of Lemma \ref{lem:Tbound}]
If $k=j+1$ then by definition and Assumption \ref{ass:linsystem},
\[ \|T_{j \rightarrow k}\| = \|I\| \leq 1.\]
Otherwise, again by definition and Assumption \ref{ass:linsystem},
\[ \|T_{j \rightarrow k}\| \leq  \left(\Pi_{t=j+2}^{k} \|A_t\|\right)  \leq (1-\delta)^{k-j-1}.\]

\end{proof}
\begin{lemma}
\label{lem:boundedstates}
Given a linear system $AB_{1:\horizon}$ satisfying Assumption \ref{ass:linsystem}, a bounded disturbance sequence $w_{1:\horizon}$ and a cost sequence $c_t$ satisfying Assumption \ref{ass:cost}, then for any $\Mdiam \geq 0, \U$, let $u_{1:\horizon} \in \U, M_{1:\horizon, 1:\gpch} \in \M_{\gamma}$ be two sequences, the following bounds hold for $x_t, a_t$ for all $t$,
\[\|x_t(u_{1:\horizon}, M_{1:\horizon, 1:\gpch})\| \leq \delta^{-1}(\lindiam\udiam  +  \lindiam\Mdiam \gpch \wdiam + \wdiam),\]
\[\|a_t(u_{1:\horizon}, M_{1:\horizon, 1:\gpch})\| \leq \udiam +  \Mdiam\gpch \wdiam.\]
Furthermore we have that for all $j,t \in [\horizon]$ we have that 
\[\biggr\|\frac{\partial x_t(u_{1:\horizon}, M_{1:\horizon, 1:\gpch})}{\partial u_j}\biggr\| \leq \begin{cases}\kappa (1-\delta)^{t-j-1} & \text{if $j < t$} \\ 0 & \text{otherwise} \end{cases}\]
\[\biggr\|\frac{\partial a_t(u_{1:\horizon}, M_{1:\horizon, 1:\gpch})}{\partial u_j}\biggr\| = \begin{cases} 1 & \text{if $j=t$}\\ 0 & \text{otherwise}\end{cases} \]
Furthermore we have that for $j,t \in [\horizon]$ and $r \in [\gpch]$, we have that 
\[\biggr\|\frac{\partial x_t(u_{1:\horizon}, M_{1:\horizon, 1:\gpch})}{\partial M_{j,r}}\biggr\| \leq \begin{cases}\lindiam \wdiam (1-\delta)^{t-j-1} & \text{if $j < t$} \\ 0 & \text{otherwise} \end{cases}\]
\[\biggr\|\frac{\partial a_t(u_{1:\horizon}, M_{1:\horizon, 1:\gpch})}{\partial M_{j,r}} \biggr\| \leq  \begin{cases}  \wdiam & \text{if $j=t$}\\ 0 & \text{otherwise}\end{cases} \]
\end{lemma}

\begin{proof}
From the definition in Equation \ref{eqn:adef} it follows that 
\begin{equation*}
	\|a_t(u_{1:\horizon}, M_{1:\horizon, 1:\gpch})\| \leq 
	\|u_{t}\| + \sum_{r=1}^{\gpch} \|M_{t,r}\| 
	\|w_{t-r}\| \leq \udiam +  \Mdiam\gpch \wdiam.  
\end{equation*}
Also from the definition it follows that 
\[\biggr\|\frac{\partial a_t(u_{1:\horizon}, M_{1:\horizon, 1:\gpch})}{\partial u_j} \biggr\|= \|\delta_{jt} I\| = \begin{cases} 1 & \text{if $j=t$}\\ 0 & \text{otherwise}\end{cases}\]
From the expansion in Equation \ref{eqn:xexpansion}, we have that
\begin{align*}
    &\|x_t(u_{1:\horizon},M_{1:\horizon,1:\gpch})\| \leq  \sum_{j=1}^{t-1} \left(\|T_{j \rightarrow t} (B_ju_j+w_j)\|\right) + \|\psi_{1\rightarrow t}^{M}(M_{1:\horizon, 1:\gpch})\| \\ 
    &\leq \sum_{j=1}^{t-1} \left(\|T_{j \rightarrow t}\| \|(B_ju_j+w_j)\|\right) + \sum_{j=1}^{t-1} \left( \|T_{j \rightarrow t}\| \left(\sum_{r=1}^{\gpch} \|M_{j,r}\| \|w_{j-r}\| \right)\right) & (\text{Definition \ref{def:linops} \& $\Delta$-inequality})\\
    &\leq (\lindiam \udiam  + \lindiam \Mdiam \gpch \wdiam + \wdiam) )\sum_{j=1}^{t-1} (1-\delta)^{t-j-1}  &(\text{Lemma \ref{lem:Tbound} and definitions} ) \\
    &\leq \frac{1}{\delta}(\lindiam \udiam  + \lindiam \Mdiam \gpch \wdiam + \wdiam)
\end{align*}
Also from the definition it follows that for $j \geq t$,
\begin{align*}
    \frac{\partial x_t(u_{1:\horizon}, M_{1:\horizon, 1:\gpch})}{\partial u_j} = 0,
\end{align*}
and if $j < t$, we have that 
\begin{align*}
    \biggr\|\frac{\partial x_t(u_{1:\horizon}, M_{1:\horizon, 1:\gpch})}{\partial u_j}\biggr\| \leq \|T_{j \rightarrow t}B_j\| \leq \kappa (1-\delta)^{t-j-1}  \qquad (\text{Lemma \ref{lem:Tbound}})
\end{align*}
From the definition in Equation \ref{eqn:adef} it follows that for any $r \in [\gpch]$ 
\[\biggr\|\frac{\partial a_t(u_{1:\horizon}, M_{1:\horizon, 1:\gpch})}{\partial M_{j,r}} \biggr\|= \|\delta_{jt} I \otimes w_{t-r}^{\top}\| \leq  \begin{cases}  \wdiam & \text{if $j=t$}\\ 0 & \text{otherwise}\end{cases}\]
From the expansion in Equation \ref{eqn:xexpansion}, it follows that for any $r$ and $j \geq t$,
\begin{align*}
    \frac{\partial x_t(u_{1:\horizon}, M_{1:\horizon, 1:\gpch})}{\partial M_{j,r}} = 0,
\end{align*}
and if $j < t$, we have that 
\begin{align*}
    \biggr\|\frac{\partial x_t(u_{1:\horizon}, M_{1:\horizon, 1:\gpch})}{\partial M_{j,r}}\biggr\| \leq \|T_{j \rightarrow t}B_j(I \otimes w_{j-r}^{\top})\| \leq \lindiam \wdiam (1-\delta)^{t-j-1}  \qquad (\text{Lemma \ref{lem:Tbound}})
\end{align*}
\end{proof}
We are now ready to prove Lemma \ref{lemma:gradbound}. 
\begin{proof}[Proof of Lemma \ref{lemma:gradbound}]
Consider the following calculations for all $j,t$, following from Lemma \ref{lem:boundedstates},
\begin{align*}
    &\|\nabla_{u_j}\left(  c_t(u_{1:\horizon}, M_{1:\horizon,1:\gpch}) \right)\| \\
    &\leq G \max(\|x_t(u_{1:\horizon}, M_{1:\horizon,1:\gpch}\|\|a_t(u_{1:\horizon}, M_{1:\horizon,1:\gpch}\|)\left(\biggr\|\frac{\partial x_t(u_{1:\horizon}, M_{1:\horizon,1:\gpch})}{\partial u_{j}}\biggr\| +  \biggr\|\frac{\partial  a_t(u_{1:\horizon}, M_{1:\horizon,1:\gpch})}{\partial u_j}\biggr\| \right) \\ 
    &\leq \begin{cases}
    G\lindiam\delta^{-1}(\lindiam\udiam  +  \lindiam\Mdiam \gpch \wdiam + \wdiam)(1-\delta)^{t-j-1} & \text{if $j < t$}\\
    G\lindiam\delta^{-1}(\lindiam\udiam  +  \lindiam\Mdiam \gpch \wdiam + \wdiam) & \text{$j=t$} \\
    0 & \text{otherwise}
    \end{cases}
\end{align*}
Therefore we have that, 
\[ \biggr\|\nabla_{u_j}\left( \sum_{t=1}^{\horizon}  c_t(u_{1:\horizon}, M_{1:\horizon,1:\gpch}) \right)\biggr\| \leq 2G\lindiam\delta^{-2}(\lindiam\udiam  +  \lindiam\Mdiam \gpch \wdiam + \wdiam)\]

\end{proof}

\subsection{Proof of Lemma \ref{lemma:gradcomparison}}

In this section we prove Lemma \ref{lemma:gradcomparison}. To this end we will need the following lemma that is the extension of Lemma \ref{lem:boundedstates} to surrogate states. 

\begin{lemma}
\label{lem:boundedstatesM}
Given a linear system $AB_{1:\horizon}$ satisfying Assumption \ref{ass:linsystem}, a bounded disturbance sequence $w_{1:\horizon}$ and a cost sequence $c_t$ satisfying Assumption \ref{ass:cost}, then for any $\Mdiam \geq 0, \U$, let $u_{1:\horizon} \in \U, M_{1:\gpch} \in \M_{\gamma}$ be two sequences, then we have that for all $j,t \in [\horizon]$,

\[\|\hat{x}_t(u_{1:\horizon}, M_{1:\gpch})\| \leq \delta^{-1}(\lindiam\udiam  +  \lindiam\Mdiam \gpch \wdiam + \wdiam)\]
Furthermore we have that for $t \in [\horizon]$ and $r \in [\gpch]$, we have that 
\[\biggr\|\frac{\partial \hat{x}_t(u_{1:\horizon}, M_{1:\gpch})}{\partial M_{r}}\biggr\| \leq \delta^{-1} \lindiam \wdiam \]
\end{lemma}

\begin{proof}
From the expansion in Equation \ref{eqn:sxdef}, we have that
\begin{align*}
    &\|x_t(u_{1:\horizon},M_{1:\gpch})\| \leq  \sum_{j=t-\gpcl}^{t-1} \left(\|T_{j \rightarrow t} (B_ju_j+w_j)\|\right) + \|\psi_{t-\gpcl \rightarrow t}^{M}(\mathcal{T}_{T} M_{ 1:\gpch})\| \\ 
    &\leq \sum_{j=t-\gpcl}^{t-1} \left(\|T_{j \rightarrow t}\| \|B_ju_j+w_j\|\right) + \sum_{j=t-\gpcl}^{t-1} \left( \|T_{j \rightarrow t}\| \left(\sum_{r=1}^{\gpch} \|M_{r}\| \|w_{j-r}\| \right)\right) & (\text{Definition \ref{def:linops} \& $\Delta$-inequality})\\
    &\leq (\lindiam\udiam  +  \lindiam\Mdiam \gpch \wdiam + \wdiam) \sum_{j=t-\gpcl}^{t-1} (1-\delta)^{t-j-1}  &(\text{Lemma \ref{lem:Tbound} and Definitions} ) \\
    &\leq \frac{1}{\delta}(\lindiam\udiam  +  \lindiam\Mdiam \gpch \wdiam + \wdiam)
\end{align*}
From the expansion in Equation \ref{eqn:sxdef}, it follows that
\begin{align*}
    \left\|\frac{\partial x_t(u_{1:\horizon}, M_{1:\horizon, 1:\gpch})}{\partial M_{r}}\right\| \leq \left\|\sum_{j = t - \gpcl}^{ t-1} T_{j \rightarrow t}B_jI \otimes w_{j-r}^{\top}\right\| \leq \delta^{-1} \lindiam \wdiam   \qquad (\text{Lemma \ref{lem:Tbound}})
\end{align*}
\end{proof}
We are now ready to prove Lemma \ref{lemma:gradcomparison}. 
\begin{proof}[Proof of Lemma \ref{lemma:gradcomparison}]
Since for any $j > t$, by Lemma \ref{lem:boundedstates}, we have that

\[\frac{\partial x_t(u_{1:\horizon}, M_{1:\horizon, 1:\gpch})}{\partial M_{j,1:\gpch}} = 0,  \frac{\partial a_t(u_{1:\horizon}, M_{1:\horizon, 1:\gpch})}{\partial M_{j,1:\gpch}} = 0, \]
it immediately follows that for all $j > t$, \[\nabla_{M_{j,1:\gpch}} c_t(u_{1:\horizon}, M_{1:\horizon,1:\gpch}) = 0.\]
Furthermore again from Lemma \ref{lem:boundedstates}, we have that for all $j < t$ and for all $r \in [\gpch]$,
\[\biggr\|\frac{\partial x_t(u_{1:\horizon}, M_{1:\horizon, 1:\gpch})}{\partial M_{j,r}}\biggr\| \leq \lindiam \wdiam (1-\delta)^{t-j-1} \]
and further if $j<t$ and for all $r \in [\gpch]$,
\[\frac{\partial a_t(u_{1:\horizon}, M_{1:\horizon, 1:\gpch})}{\partial M_{j,r}} = 0  \]
Therefore, since the cost function $c_t$ satisfies the Assumption \ref{ass:cost}, using Lemma \ref{lem:boundedstates}, we have that for all $j < t$ and for any $r \in [\gpch]$
\begin{align}
    \biggr\|\nabla_{M_{j,r}} c_t(u_{1:\horizon}, M_{1:\horizon, 1:\gpch})\biggr\| &\leq G \|x_t(u_{1:\horizon}, M_{1:\horizon, 1:\gpch})\|  \biggr\|\frac{\partial x_t(u_{1:\horizon}, M_{1:\horizon, 1:\gpch})}{\partial M_{j,r}}\biggr\| \nonumber\\
    &\leq  G\lindiam\delta^{-1} \wdiam (\lindiam\udiam  +  \lindiam\Mdiam \gpch \wdiam + \wdiam) (1-\delta)^{t-j} \label{eqn:costgradbound}
\end{align}
Using Lemma \ref{lem:boundedstatesM} for the surrogate states and using Assumption \ref{ass:cost}, we have that for all $t$ and for all $r \in [\gpch]$,
\begin{align*}
    \|\nabla_{M_r} \hat{c}_t(u_{1:\horizon}, M_{1:\gpch})\|
    &\leq 2G\lindiam\delta^{-2}\wdiam(\lindiam\udiam  +  \lindiam\Mdiam \gpch \wdiam + \wdiam)
\end{align*}
Since the gradient is bounded according to the above calculation and the $M_{t,1:\gpch}$ are generated via gradient descent with a learning rate $\eta_{\mathrm{in}}$, it is immediate that for any  $j,k \in [T]$ and for any $r \in [\gpch]$ ,
\begin{equation}
\label{eqn:Mclose}
    \| M_{j,r} - M_{k,r} \| \leq \eta_{\mathrm{in}}|j-k| \cdot 2G\lindiam\delta^{-2}\wdiam(\lindiam\udiam  +  \lindiam\Mdiam \gpch \wdiam + \wdiam)
\end{equation}
Given the above we show that for any execution the surrogate states and the real states are close to each other. To this end consider the following calculations.
\begin{align}
   &\|x_t(u_{1:\horizon}, M_{1:\horizon,1:\gpch}) -  \hat{x}_t(u_{1:\horizon}, M_{t,1:\gpch})\| \nonumber\\
   &\leq \biggr\|\sum_{j=1}^{t-1} T_{j \rightarrow t} \left(B_ju_j + w_j\right) + \psi_{1\rightarrow t}^{M}(M_{1:\horizon, 1:\gpch})  - \sum_{j=t-\gpcl}^{t-1} T_{j\rightarrow t} \left(B_ju_j + w_j\right) -  \psi^M_{t-\gpcl \rightarrow t}(\mathcal{T}_{\horizon} M_{t,1:\gpch})\biggr\|\nonumber\\
   &= \biggr\|\sum_{j=1}^{t-\gpcl-1} \left( T_{j \rightarrow t} \left(B_ju_j + w_j +  \sum_{r=1}^{\gpch} M_{j,r} w_{j-r} \right)\right)+ \sum_{j=t-\gpcl}^{t-1} \left( T_{j \rightarrow t} \left(\sum_{r=1}^{\gpch} (M_{j,r} - M_{t,r}) w_{j-r} \right)\right)\biggr\|\nonumber\\
   &\leq (\lindiam\udiam + \lindiam\Mdiam\gpch \wdiam + \wdiam) \left(\delta^{-1}(1 - \delta)^{\gpcl} + 2\eta_{\mathrm{in}}\lindiam \delta^{-2}\gpcl^2 G \gpch \wdiam^2 \right) \label{eqn:surrogatediff}
\end{align}
Furthermore, note by definitions that 
\begin{align}
    \sum_{j=t-\gpcl}^{t-1} \frac{\partial x_t(u_{1:\horizon}, M_{1:\horizon, 1:\gpch})}{\partial M_{j,1:\gpch}} = \frac{\partial \hat{x}_t(u_{1:\horizon}, M_{t, 1:\gpch})}{\partial M_{t,1:\gpch}} \label{eqn:derequal}
\end{align}
Before moving further, consider the following calculations
\begin{align*}
    &\sum_{j=t-\gpcl}^{t-1} \left(\nabla_{M_{j,1:\gpch}} c_t(u_{1:\horizon}, M_{1:\horizon,1:\gpch})\right) \\
    &=\sum_{j=t-\gpcl}^{t-1} \left( \frac{\partial x_t(u_{1:\horizon}, M_{1:\horizon, 1:\gpch})}{\partial M_{j,1:\gpch}} \nabla_x  c_t(x_t(u_{1:\horizon}, M_{1:\horizon,1:\gpch}),a_t(u_{1:\horizon}, M_{1:\horizon,1:\gpch}))  \right)\\
    &= \sum_{j=t-\gpcl}^{t-1}  \frac{\partial x_t(u_{1:\horizon}, M_{1:\horizon, 1:\gpch})}{\partial M_{j,1:\gpch}} \left(\left(\nabla_x  c_t(\hat{x}_t(u_{1:\horizon}, M_{t,1:\gpch}),a_t(u_{1:\horizon}, M_{1:\horizon,1:\gpch})) + v\right)\right)
\end{align*}
where 
\begin{multline}
\label{eqn:vbound}
    \|v\| \defeq \|\nabla_x  c_t(x_t(u_{1:\horizon}, M_{1:\horizon,1:\gpch}),a_t(u_{1:\horizon}, M_{1:\horizon,1:\gpch})) - \nabla_x  c_t(\hat{x}_t(u_{1:\horizon}, M_{t,1:\gpch}),a_t(u_{1:\horizon}, M_{1:\horizon,1:\gpch}))\| \\
    \leq \beta (\lindiam\udiam + \lindiam\Mdiam\gpch \wdiam + \wdiam) \left(\delta^{-1}(1 - \delta)^{\gpcl} + 2\eta_{\mathrm{in}}\lindiam \delta^{-2}\gpcl^2 G \gpch \wdiam^2 \right)
\end{multline}
using Equation \ref{eqn:surrogatediff} and the $\beta$-smoothness of $c_t$ via Assumption \ref{ass:cost}. Using Equation \ref{eqn:derequal} and Lemma \ref{lem:boundedstates} we now get that 
\begin{align}
    &\sum_{j=t-\gpcl}^{t-1} \left(\nabla_{M_{j,1:\gpch}} c_t(u_{1:\horizon}, M_{1:\horizon,1:\gpch})\right) \nonumber\\
    &= \left(\sum_{j=t-\gpcl}^{t-1}  \frac{\partial x_t(u_{1:\horizon}, M_{1:\horizon, 1:\gpch})}{\partial M_{j,1:\gpch}}\right) \left(\nabla_x  c_t(\hat{x}_t(u_{1:\horizon}, M_{t,1:\gpch}),a_t(u_{1:\horizon}, M_{1:\horizon,1:\gpch})) + v\right) \nonumber\\
    &=\frac{\partial \hat{x}_t(u_{1:\horizon}, M_{1:\horizon, 1:\gpch})}{\partial M_{t,1:\gpch}}  \left(\nabla_x  c_t(\hat{x}_t(u_{1:\horizon}, M_{t,1:\gpch}),a_t(u_{1:\horizon}, M_{1:\horizon,1:\gpch}))\right) + v' \label{eqn:xgradcomparison}
\end{align}
 where $v'$ is a vector whose norm using Equation \ref{eqn:vbound} and Lemma \ref{lem:boundedstatesM} can be bounded as follows 
 \begin{equation}
 \label{eqn:vprimebound}
     \beta \delta^{-1}\gpch \wdiam (\lindiam\udiam + \lindiam\Mdiam\gpch \wdiam + \wdiam) \left(\delta^{-1}(1 - \delta)^{\gpcl} + 2\eta_{\mathrm{in}}\lindiam \delta^{-2}\gpcl^2 G \gpch \wdiam^2 \right).
 \end{equation}
Now, consider the following computation which follows from Equation \ref{eqn:xgradcomparison} and using the defintiions for the $j=t$ case,
\begin{align}
    &\sum_{j=t-\gpcl}^{t} \left(\nabla_{M_{j,1:\gpch}} c_t(u_{1:\horizon}, M_{1:\horizon,1:\gpch})\right)
    =\nabla_{M_{t,1:\gpch}} \hat{c}_t(u_{1:\horizon}, M_{t,1:\gpch}) + v'  \label{eqn:fullgradcomparison}.
\end{align}
We can now perform the calculation to relate the gradient inner products for surrogate cost to those of real cost. 
\begin{align}
     &\sum_{j=t-\gpcl}^t \left(\nabla_{M_{j,1:\gpch}} c_t(u_{1:\horizon}, M_{1:\horizon,1:\gpch}) (M_{j,1:\gpch} - \Mstar_{1:\gpch})\right) \nonumber\\  
&= \sum_{j=t-\gpcl}^t \left( \nabla_{M_{j,1:\gpch}} c_t(u_{1:\horizon}, M_{1:\horizon,1:\gpch}) (M_{t,1:\gpch} -  \Mstar_{1:\gpch}) + \nabla_{M_{j,1:\gpch}} c_t(u_{1:\horizon}, M_{1:\horizon,1:\gpch}) (M_{j,1:\gpch} -  M_{t,1:\gpch}) \right) \nonumber\\
&\leq \sum_{j=t-\gpcl}^t \left( \nabla_{M_{j,1:\gpch}} c_t(u_{1:\horizon}, M_{1:\horizon,1:\gpch}) (M_{t,1:\gpch} -  \Mstar_{1:\gpch}) \right) + \eta_{\mathrm{in}} 2G^2\gpch\gpcl^2\lindiam^2\delta^{-3}\wdiam^2(\lindiam\udiam  +  \lindiam\Mdiam \gpch \wdiam + \wdiam)^2\nonumber \\
&\leq  \nabla_{M_{t,1:\gpch}} \hat{c}_t(u_{1:\horizon}, M_{t,1:\gpch}) (M_{t,1:\gpch} -  \Mstar_{1:\gpch}) + \nonumber\\
&\qquad \qquad \qquad  \beta \delta^{-1}\Mdiam\gpch^2 \wdiam (\lindiam\udiam + \lindiam\Mdiam\gpch \wdiam + \wdiam)^2 \left(\delta^{-1}(1 - \delta)^{\gpcl} + 4\eta_{\mathrm{in}}\lindiam^2 \delta^{-2}\gpcl^2 G^2 \gpch \wdiam^2 \right) \nonumber\\
&\leq  \nabla_{M_{t,1:\gpch}} \hat{c}_t(u_{1:\horizon}, M_{t,1:\gpch}) (M_{t,1:\gpch} -  \Mstar_{1:\gpch})\nonumber + 5\eta_{\mathrm{in}}\log^2(\eta_{\mathrm{in}}) \Mdiam\lindiam^2 \delta^{-3} \beta G^2 \gpch^3 \wdiam^3(\kappa\udiam + \kappa \Mdiam \gpch \wdiam + \wdiam)^2 \nonumber
\end{align}

where the first inequality follows from applying Equations \ref{eqn:costgradbound}, \ref{eqn:Mclose} and Lemma \ref{lem:boundedstates}, the second last inequality follows from Equations \ref{eqn:vprimebound} and \ref{eqn:fullgradcomparison} and the last inequality follows from the choice of the parameter $\gpcl = \delta^{-1}\log(\eta_{\mathrm{in}})$. This finishes the proof.
\end{proof}


\section{Adaptation of Algorithm to General Policies}
\label{app:generalalgos}

In this section we provide a more general version of our algorithms \ref{alg:bigalg} and \ref{alg:gpcoverlay}, defined for any base outer policy class $\Pi$. Note that our formal results dont cover this generalization and it is provided with practical use in mind. 
\clearpage

\begin{algorithm}[h!]
\caption{iGPC Algorithm}
\label{alg:bigalgfull}
\begin{algorithmic}[1]
\Require [Online] $f_{1:\horizon}^{1:N}:$ Dynamical Systems, $w_{1:\horizon}^{1:N}:$ Disturbances 
\Params Policy class: $\Pi$, $\eta_{\mathrm{out}}:$ Learning Rate
\vspace{5pt} 
\State Initialize $\pi_{1:\horizon}^1 \in \Pi$.
\For{$i = 1\ldots N$}
\State Receive the dynamical system $f_{1:\horizon}^i$ for the next rollout. 
\State \textbf{Rollout}: Collect trajectory data by rolling out policy $\pi^i_{1:\horizon}$ with GPC \Comment{(Algorithm \ref{alg:gpcoverlay})}
\[ \mathrm{TrajData}^i = \{ x^i_{1:\horizon}, a^i_{1:\horizon}, w^i_{1:\horizon}, o_{1:\horizon}^i\} \leftarrow \mathrm{GPCRollout}(f_{1:\horizon}^i, \pi^i_{1:\horizon}) \] 
\State \textbf{Update}: Compute update to the policy
\[\pi_{1:\horizon}^{i+1} = \mathrm{Proj}_{\Pi} \left( \pi_{1:\horizon}^{i} - \eta_{\mathrm{out}} \nabla_{\pi_{1:\horizon}} J( \pi_{1:\horizon}^{i}+ \pi(o_{1:\horizon}^{i}) | f_{1:\horizon}^i, w_{1:\horizon}^i)\right)\]
\EndFor
\end{algorithmic}
\end{algorithm}

\begin{algorithm}[h!]
\caption{GPCRollout}
\label{alg:gpcoverlayfull}
\begin{algorithmic}[1]
\Require $f_{1:\horizon}$: dynamical system, $\pi_{1:\horizon}$: input policy, [Online] $w_{1:\horizon}$: disturbances. 
\Params  $\gpch$:Window, $\eta_{\mathrm{in}}$: Learning rate, $\Mdiam$: Feedback bound, $\gpcl$: Lookback
\vspace{5pt}
\State Initialize $M_{1,1:\gpch} = \{M_{1,j}\}_{j=1}^{\gpch} \in \M_{\Mdiam}$.  
\State Set $w_i = 0$ for any $i \leq 0$.
\For{$t=1 \ldots \horizon$}
\State Compute GPC Offset \[o_t = M_{t,1:\gpch} \cdot w_{t-1:t-\gpch}.\] 
\State Play action \[a_{t} = \pi_t(\cdot) + o_t\]
\State Observe state $x_{t+1}$.
\State Compute perturbation \[w_{t} = x_{t+1} - f_t(x_t, a_t).\]
\State Update $M_{t+1, 1:\horizon}$ for the next round as:
\[ M_{t+1, 1:\gpch} = \mathrm{Proj}_{\M_{\kappa}} \left(M_{t, 1:\gpch} - \eta_{\mathrm{inner}} \nabla_{M_{1:\gpch}} \mathrm{GPCLoss}(M_{t, 1:\gpch}, \pi_{t-\gpcl+1:t}, w_{t-\gpcl-\gpch+1:t-1} )\right)\]
\Comment{$\mathrm{GPCLoss}$ defined in Equation \ref{eqn:gpcloss}}
\EndFor
\State \Return $x_{1:\horizon}, a_{1:\horizon}, w_{1:\horizon}$, $o_{1:\horizon}$.
\end{algorithmic}
\end{algorithm}

\section{Details of ILQR/ILC/IGPC Algorithms}
\label{app:allalgoscomb}
To succinctly state the algorithms define the following policy which takes as arguments a nominal trajectory $\mathring{x}_{1:\horizon} \in \reals^{d_x},\mathring{u}_{1:\horizon} \in \reals^{d_u}$, open-loop gain seqeunce $k_{1:\horizon}$ and closed-loop gain sequence $K_{1:\horizon}$ and a parameter $\alpha$. The policy defined as $\pi(\alpha, x_{1:\horizon},k_{1:\horizon},K_{1:\horizon})$, in the sequel executes the following \textit{standard} rollout on a dynamical system $f_{1:\horizon}$.

\[a_t = \mathring{u}_t + \alpha k_t + K(x_{t-1} - \mathring{x}_{t-1}) \]
\[x_{t+1} = f_t(x_t, a_t) \]
Before stating the algorithm we also need the following quadratic approximation of the cost function $c$ around pivots $x_0, u_0$
\begin{multline}
\label{eqn:quadraticapprox}
	Q(c,x_0,u_0)(x,u) \defeq \nabla c_x(x_0,u_0)(x-x_0) + \nabla c_u(x_0,u_0)(u-u_0) \\ + \frac{1}{2} ([x,u] - [x_0,u_0])^{\top} \nabla^2 c(x,u) ([x,u] - [x_0,u_0])
\end{multline}
Algorithm \ref{alg:bigalgfullgeneric} now presents a combined layout for ILQG,ILC and IGPC. 
\begin{algorithm}[h!]
\caption{Iterative Planning Algorithm}
\label{alg:bigalgfullgeneric}
\begin{algorithmic}[1]
\Require $g_{1:\horizon}$ Real Dynamical Systems, $g_{1:\horizon}$ Simulator.
\vspace{5pt} 
\State Initialize starting sequence of actions $u^{0}_{1:\horizon}$
\State Initialize sequence of open loop $k_{1:\horizon}^{0} = 0$ and closed loop gains $K_{1:\horizon}^{0} = 0$.
\For{$i = 1\ldots N$}
\State \textbf{Rollout the Policy}:
\begin{itemize}
	\item \textbf{ILQG:} Standard Rollout on $f_{1:\horizon}$.
	\[ x^{i}_{1:\horizon}, u^{i}_{1:\horizon} = \mathrm{Rollout}(f_{1:\horizon}, \pi(\alpha, x_{1:\horizon}^{i-1}, u^{i-1}_{1:\horizon}, k^{i-1}_{1:\horizon}, K^{i-1}_{1:\horizon}))\]
	\item \textbf{ILC:} Standard Rollout on $g_{1:\horizon}$.
		\[ x^{i}_{1:\horizon}, u^{i}_{1:\horizon} = \mathrm{Rollout}(g_{1:\horizon}, \pi(\alpha, x_{1:\horizon}^{i-1}, u^{i-1}_{1:\horizon}, k^{i-1}_{1:\horizon}, K^{i-1}_{1:\horizon}))\].
	\item \textbf{IGPC:} GPCRollout on $g_{1:\horizon}$, \[x^{i}_{1:\horizon}, u^{i}_{1:\horizon} = \mathrm{GPCRollout}(g_{1:\horizon},\pi(\alpha, x_{1:\horizon}^{i-1}, u^{i-1}_{1:\horizon}, k^{i-1}_{1:\horizon}, K^{i-1}_{1:\horizon}))\]
\end{itemize} 
\State \textbf{Update}: Obtain $k^{i}_{1:\horizon} \in \reals^{d_u}, K^{i}_{1:\horizon} \in \reals^{d_u \times d_x}$ as the optimal non-stationary affine policy to the following LQG problem. 

\[ \min \mathbb{E}_{z} \left[ \sum_{t=1}^{\horizon} Q(c_t, x^{i}_t,u^{i}_t)(x_t, u_t) \right]\]
\[\text{subject to} \qquad  x_{t+1} - x^{i}_{t+1} = \frac{\partial{f_t(x^{i}_t, u^{i}_t)}}{\partial{{x^{i}_t}}} (x_t - x^{i}_t) + \frac{\partial{f_t(x^{i}_t, u^{i}_t)}}{\partial{{u^{i}_t}}} (u_t - u^{i}_t) + z_t \]
where $z_t$ are independent Gaussians of any non-zero variance. 
\EndFor
\end{algorithmic}
\end{algorithm}

\subsection{Hyperparameter Selection for Experiments}

ILQG, ILC, IGPC in particular share one hyperparameter $\alpha$ which corresponds essentially to a step size towards the updated policy. As is common in implementations, this hyperparameter is adjusted online during the run of the algorithm using a simple retracting line search from a certain upper bound $\alpha^{+}$. We optimize over choices for $\alpha^+$ for ILC and report the best performance obtained as baseline. For IGPC, we use the same $\alpha^{+}$ as obtained for ILC and the same line search for strategy for selecting $\alpha$. We include the rollouts needed for line search in the rollout cost of the algorithm. Further, IGPC introduces certain other hyperparameters, $L$ the window, $S$ the lookback, and $\eta_{\mathrm{in}}$, the inner learning rate. We chose $L,S = 3$ arbitrarily for our experiments and tuned $\eta_{\mathrm{in}}$ per experiment. Overall we observed that for every experiment, the selection of $\eta$ was robust in terms of performance.

\end{document}